\documentclass[11pt]{article}
\usepackage[mathscr]{eucal}
\usepackage[utf8]{inputenc} 
\usepackage[T1]{fontenc}    
\usepackage{hyperref}       
\usepackage{url}            
\usepackage{booktabs,natbib}       
\usepackage{amsfonts,bm}       
\usepackage{nicefrac}       
\usepackage{microtype}      
\usepackage{pifont,xspace,paralist,times,amsthm}
\usepackage{titlesec}
\usepackage{amsmath,amssymb,times,algorithmic,fullpage}
\usepackage{latexsym,paralist,multirow,tablefootnote}
\usepackage{hyperref}
\hypersetup{colorlinks=true,linkcolor=blue,citecolor=blue}
\usepackage{graphicx}
\usepackage{amsfonts}
\usepackage{makecell}
\usepackage{caption}
\newenvironment{CompactEnumerate}{
\begin{list}{\arabic{enumi}.}{%
\usecounter{enumi}
\setlength{\leftmargin}{12pt}
\setlength{\itemindent}{3pt}
\setlength{\topsep}{3pt}
\setlength{\itemsep}{1pt}
}}
{\end{list}}
\newenvironment{CompactItemize}{
\begin{list}{$\bullet$}{%
\setlength{\leftmargin}{12pt}
\setlength{\itemindent}{1pt}
\setlength{\topsep}{1pt}
\setlength{\itemsep}{1pt}
}}
{\end{list}}

\newtheorem{theorem}{Theorem}[section]
\newtheorem*{non-theorem}{Informal Theorem}
\newtheorem{lemma}[theorem]{Lemma}

\newtheorem{claim}[theorem]{Claim}
\newtheorem{remark}[theorem]{Remark}

\newtheorem{corollary}[theorem]{Corollary}
\newtheorem{definition}{Definition}

\def \y {\mathbf y}
\def \eps {\epsilon}

\def \co {\mathrm co}

\def \hinge {\mathrm{hinge}}

\def \z {\mathbf z}

\def \NNN {\mathcal{N}}

\def \u {\mathbf u}
\def \N {\mathbb N}

\def \offdiag {{\mathrm{offdiag}}}

\def \supp {\mathrm{supp}}

\def \DisC {\mathrm{DisC}}
\def \a {\mathbf a}
\def \b {\mathbf b}
\def \c {\mathbf c}

\def \e {\mathbf e}

\def \s {\mathbf s}

\def \r {\mathbf r}
\def \v {\mathbf v}
\def \hinge {\mathrm{hinge}}
\def \x {\mathbf x}
\def \w {\mathbf w}

\def \m {\mathbf m}

\def \R {\mathbb{R}}
\def \N {\mathbb{N}}

\def \MM {\mathcal{M}}

\def \NN {\mathcal{N}}

\DeclareMathOperator{\E}{\mathbb{E}}

\DeclareMathOperator{\diag}{diag}

\DeclareMathOperator{\sr}{\mathbf{sr}}
\DeclareMathOperator{\tr}{tr}
\DeclareMathOperator{\rank}{rank}

\def \HS {\mathrm{F}}
\def \< {\left \langle}
\def \> {\right \rangle}
\def \etc {,\ldots,}
\newcommand{\norm}[1]{\left \| #1 \right \|}
\newcommand{\Prob}[2][]{\mathrm{Pr}_{#1} \left[ #2 \rule{0mm}{3mm}\right]}

\def \< {\left \langle}
\def \> {\right \rangle}

\begin{document}
\title{Restricted Isometry Property under High Correlations}
 \author{Shiva Prasad Kasiviswanathan\thanks{Amazon, USA. \href{mailto:kasivisw@gmail.com}{kasivisw@gmail.com}.} \and Mark Rudelson\thanks{University of Michigan, Ann Arbor, MI, USA. \href{mailto:rudelson@umich.edu }{rudelson@umich.edu }.  Partially supported by NSF grant, DMS-1464514.}}
\date{}
\maketitle
\begin{abstract}
 \normalsize
Matrices satisfying the Restricted Isometry Property (RIP) play an important role in the areas of compressed sensing and statistical learning. RIP matrices with optimal parameters are mainly obtained via probabilistic arguments, as explicit constructions seem hard. In this paper, we try to bridge this gap between random and deterministic designs by introducing a new model for restricted isometry designs that incorporates a fixed matrix into the construction.
Our construction starts with a fixed (deterministic) matrix $X$ satisfying some simple stable rank condition, and we show that the matrix $XR$, where $R$ is a random matrix drawn from various popular probabilistic models (including, subgaussian, sparse, low-randomness, satisfying convex concentration property), satisfies the RIP with high probability. These theorems have various applications in signal recovery, deep learning, random matrix theory, dimensionality reduction, etc. Additionally, motivated by an application for understanding the effectiveness of word vector embeddings popular in natural language processing and machine learning applications, we investigate the RIP of the matrix $XR^{(\ell)}$ where $R^{(\ell)}$ is formed by taking all possible (disregarding order) $\ell$-way entrywise products of the columns of a random matrix~$R$.

%
\end{abstract}

\section{Introduction} \label{sec:intro}
A vector $\x \in \R^d$ is said to be $k$-sparse if it has at most $k$ nonzero coordinates.  Sparsity is a structure of wide applicability, with a broad literature dedicated to its study in various scientific fields (see, e.g.,~\citep{foucart2017mathematical,eldar2012compressed}). Given an $\eps \in (0,1)$, an $n \times d$ matrix $M$ (typically with $n << d$) is said to satisfy the $(k,\eps)$-Restricted Isometry Property (RIP)~\citep{candes2005decoding} if it approximately preserves the Euclidean norm in the following sense: for every $k$-sparse vector, we have
$$ (1-\eps) \| \x \| \leq \| M \x \| \leq (1+\eps) \| \x\|.$$
RIP is a fundamental property of a matrix that enables recovery of a sparse high-dimensional signal from its compressed measurement. Given this, matrices satisfying the restricted isometry property have found many interesting applications in high-dimensional statistics, machine learning, and compressed sensing~\citep{foucart2017mathematical,eldar2012compressed,wainwright2019high}.  Restricted isometry is also closely related to other matrix properties such as restricted nullspace, restricted eigenvalue, and pairwise incoherence~\citep{wainwright2019high}.

Various probabilistic models are known to generate random matrices that satisfy the restricted isometry property with a value of $k$ which is (almost) linear $n$.  For example, generating entries of $M$ i.i.d.\ from a common distribution (like satisfying subgaussianity) and then normalizing the columns of $M$ to unit norm, guarantees RIP with high probability provided $n = \Omega((k/\eps^2) \log(d/k))$~\citep{baraniuk2008simple}.  
We refer the reader to~\citep{V11} for additional references to the probabilistic RIP literature.  The restricted isometry property also holds for a rich class of structured random matrices,  where usually the best known bounds for $n$ have additional log factors in $d$~\citep{foucart2017mathematical}. The use of randomness still remains pivotal for near-optimal results. 

At the same time, verifying whether a given matrix satisfies the restricted isometry property is tricky, as the problem of certifying RIP of a matrix in the worst case is NP-hard~\citep{bandeira2013certifying}. Furthermore, even determining the RIP value $\eps$ up to a certain approximation factor is hard in the average-case sense, as shown by~\citep{wang2016average} using a reduction from the Planted Clique Problem. 
While there are constructions of RIP matrices when $k = O(\sqrt{n})$, most methods however break down when $k$ is at least some constant times $\sqrt{n}$, see~\citep{bandeira2013road} for a survey of deterministic methods. The best unconditional explicit construction to date is due to~\citet{bourgain2011explicit}  which gives a RIP guarantee for $k = \Theta(n^{1/2+\delta})$ for some unspecified small constant $\delta > 0$, in the regime $n = \Omega(d^{1-\eps})$ and with matrix containing complex valued entries.~\citet{gamarnik2018explicit} recently showed why explicit RIP matrix construction is a ``hard'' challenge, by connecting it to a question in the field of extremal combinatorics which has been open for decades.

So on the one hand while it is easy to generate matrices satisfying the restricted isometry property through random designs, designing natural families of deterministic matrices satisfying RIP seem to run into hard barriers. The main conceptual contribution of this paper is a step towards bridging this gap. We construct families of practically motivated not-completely random matrices that satisfy RIP. We focus on a natural way of incorporating a fixed (deterministic) matrix into RIP constructions, raising the question of whether there are other interesting models that lie in the intersection of random and deterministic designs.
%

\smallskip
\noindent\textbf{Our Contributions.}  We establish the restricted isometry property for a wide class of matrices, which can be factorized through (possibly non-i.i.d.) random matrices. In particular, we will be interested in the class of matrices which have a  $XR$-factorization, where $X$ is a fixed (deterministic) $n \times p$ matrix and $R$ is a $p \times d$ random matrix. The $XR$-model (product of a deterministic and random matrix) in the context of RIP captures a variety of applications some of which we discuss later. The main challenge in establishing RIP comes from the fact that the entries in the matrix $XR$ could be highly correlated, even if the entries in $R$ are independent.\!\footnote{Note that w.r.t.\ RIP condition, the $XR$-model behaves quite differently from the $RX$-model as we explain in Appendix~\ref{app:prelim}.}.

Our main result is that if we start with any deterministic matrix $X$, satisfying a very mild easy to check condition, the matrix $XR$ satisfies RIP (with high probability) for a $R$ constructed from a variety of popular probabilistic models. All we need is that the stable rank (or numerical rank) of $X$ is not ``too small''. Stable rank of a matrix $X$ (denoted by $\sr(X)$), defined as the squared ratio of Frobenius and spectral norms of $X$ is a commonly used robust surrogate to usual matrix rank in linear algebra. Stable rank of a matrix is at most its usual rank. Computing the stable rank of a matrix is a polynomial time operation, meaning that given the factorization $XR$ one could easily verify whether the required conditions on $X$ is satisfied. We investigate many common constructions of the random matrix~$R$:\footnote{All the results are high probability statements, and for simplicity we omit the dependence on certain parameters such as $\eps$, $\psi_2$-norm of subgaussian vectors, etc.}  
\begin{CompactItemize} \label{col:a}
\item {\em Columns of $R$ are independent subgaussian random vectors:} In this setting, we obtain RIP on the $XR$ matrix, if the stable rank of $X$ satisfies $\sr(X) =\Omega(k \log(d/k))$ (see Theorem~\ref{thm:rand}). Note that this setting includes the (standard random matrix) case where $R$ is an i.i.d.\ subgaussian random matrix. The dependence on $d,k$ in this stable rank condition cannot be improved in general.
\item {\em Generated by $l$-wise independent distributions:} We ask: can we reduce the amount of randomness in $R$? We answer in affirmative by showing that one can achieve restricted isometry with the same condition on $\sr(X)$ as above, by only requiring that $R$ be generated from a $2\sr(X)$-wise independent distribution (see Theorem~\ref{thm:lownoise}). 
\item {\em Sparse-subgaussian matrix:} Sparsity is a desirable property in $R$ as it leads to computational speedup when working with sparse signals. We use the standard Bernoulli-Subgaussian process for generating a sparse matrix $R$, where $R$ is defined as the Hadamard matrix product of an i.i.d.\ Bernoulli$(\beta)$ random matrix and an i.i.d.\ subgaussian random matrix. In this case, we get RIP if the stable rank of $X$ satisfies, $\sr(X) =\Omega(k \log(d/k)/\beta)$, where $\beta$ is the Bernoulli parameter (see Theorem~\ref{thm:sparse}).
\item {\em Columns of $R$ are independent vectors satisfying convex concentration:} Our result holds for distributions that satisfy the so-called convex concentration property. The convex concentration property of a random vector was first observed by Talagrand who first proved it for the uniform measure on the discrete cube and for general product measures with bounded support~\citep{talagrand1988isoperimetric,talagrand1995concentration}. Vectors satisfying convex concentration are regularly used in statistical analysis, as it includes random vectors drawn from a centered multivariate Gaussian distribution with arbitrary covariance matrix, random vectors uniformly distributed on the sphere,  random vectors satisfying the {\em logarithmic Sobolev} inequality, among many others~\citep{ledoux2001concentration}. Ignoring the dependence on the convex concentration constant, again we get RIP on $XR$ matrix if the stable rank of $X$ satisfies, $\sr(X) =\Omega(k \log(d/k))$ (see Theorem~\ref{thm:ccp}). 
\item {\em $\ell$-way column Hadamard-product construction}: Motived by an application in understanding the effectiveness of word vector embeddings (also referred to as just word embeddings) in linear classification tasks, we investigate a correlated random matrix, formed by taking all possible (disregarding order) $\ell$-way entrywise products\footnote{The entrywise product of $\ell$ vectors $\v_1,\dots,\v_\ell \in \R^p$ is the vector $\hat{\v} \in \R^p$ with $j$ entry equaling $\hat{v}_j = \prod_{i=1}^{\ell} v_{i_j}$.} of the columns of an  i.i.d.\ centered bounded random matrix $R$ (see Definition~\ref{defn:3}). Let $R^{(\ell)} \in \R^{p \times \binom{d}{\ell}}$ denote this constructed matrix starting from $R$ (with $R^{(1)} = R$).  We establish RIP on $XR^{(\ell)}$ for $\ell = 2$ if $\sr(X) = \Omega(k^2 \log(d^2/k))$ (see Theorem~\ref{thm: RIP}), and for $\ell \geq 3$ if $\sr(X) = \Omega(k^3 \log(d^\ell/k))$ (see Theorem~\ref{thm: RIP2}). Notice that the dependence on $\sr(X)$ on the sparsity parameter $k$ is worse than in the previous cases. However, the value of $k$ used in the motivating application is typically small. 
\end{CompactItemize}

Our proofs rely on the Hanson-Wright inequality which provides a large deviation bound for quadratic forms of i.i.d.\ subgaussian random variables~\citep{RVHanson-Wright}, along with its recent extensions~\citep{zhou2015sparse,adamczak2015note}. While the proof is simple when the columns of $R$ are independent subgaussian random vectors, various challenges arise when dealing with other models of the random matrix.  One general idea is get to a concentration bound on $\|X Z \u\|$, where $X$ is a fixed matrix, $Z$ is some random matrix, and $\u$ is a fixed $k$-sparse vector, and then use a net argument over sparse vectors on the sphere. In the case of column Hadamard-product, getting this concentration bound is tricky as it involves analyzing some high order homogeneous chaos in terms of the random variables, and we use a different idea based on bounding the $\psi_2$-norm. 

Throughout the proofs the challenge comes in dealing with dependences that arise both in the random matrix, and in the product matrix. This is especially true for the Hadamard-product construction, where {\em without} $X$ (i.e., $X$ is the identity matrix), the problem is much simpler and the matrix $R^{(\ell)}$ satisfies RIP under a milder condition that $p=\Omega(k \log(d^\ell/k))$~\citep{rudelson2008sparse,arora2018compressed}. Intuitively in the $R^{(\ell)}$ case one has only to ensure that $\|R^{(\ell)} \u \|$ is close to  a constant with high probability for a fixed sparse unit vector $\u$. In the case of $XR^{(\ell)}$, in addition to it, one has to guarantee that the direction of the vector $R^{(\ell)} \u$ is more or less uniformly distributed over the sphere, making the problem significantly harder to analyze.
So while it is tempting to conjecture that the right dependence of $k$ in $\sr(X)$ in the Hadamard-product case should also be linear, it appears challenging to obtain this linear bound, and it is plausible that this superlinear dependence on $k$ might in fact be unavoidable.



\smallskip
\noindent\textbf{Applications.} The restricted isometry property is widely utilized in compressed sensing and statistical learning literature. Here we mention a few interesting applications of our restricted isometry results. 
\begin{list}{{\bf (\arabic{enumi})}}{\usecounter{enumi}
\setlength{\leftmargin}{8pt}
\setlength{\listparindent}{3pt}
\setlength\parindent{0pt}
\setlength{\parsep}{3pt}}
\item\textbf{Effectiveness of Word/Sequence Embeddings.} \label{app1} 
Consider a vocabulary set $\mathcal{W}$ (say, all words in a particular language). Word embeddings, which associates with each word $\w$ in a vocabulary set $\mathcal{W}$ a vector representation $\v_{w} \in \R^p$, is a basic building block in Natural Language Processing pipelines and algorithms. Word embeddings were recently popularized via embedding schemes such as {\em word2vec}~\citep{mikolov2013distributed} and {\em GloVe}~\citep{pennington2014glove}. Word embeddings pretrained on large sets of raw text have demonstrated remarkable success when used as features to a supervised learner in various applications ranging from question-answering~\citep{zhou2015learning} to sentiment classification of text documents~\citep{maas2011learning}. The general intuition behind word embeddings is that they are known to capture the ``similarities'' between words. There has also been recent work on creating representations for {\em word sequences} such as {\em phrases} or {\em sentences} with methods ranging from simple composition of the word vectors to sophisticated architectures as {\em recurrent neural networks} (see e.g.,~\citep{arora2016simple} and references therein)

\noindent Understanding the theoretical properties of these word embeddings is an area of active interest. In a recent result,~\citet{arora2018compressed} introduced the scheme of {\em Distributed Cooccurrence} (DisC) embedding (Definition~\ref{defn:2}) for a word sequence that produces a compressed representation of a Bag-of-$L$-cooccurrences vector.\footnote{An $\ell$-cooccurrence is a set of $\ell$ words. Bag-of-$L$-cooccurrences for a word sequence counts the number of times any possible $\ell$-cooccurrence (from the vocabulary set $\mathcal{W}$) for $\ell \in [L]$ appears in the word sequence. See Definition~\ref{defn:1}. Linear classification models are empirically known to perform well over these simple representations~\citep{wang2012baselines,arora2018compressed}.}
They showed that if one uses i.i.d.\ $\pm 1$-random vectors as word embeddings, then a linear classifier trained on these compressed DisC embeddings performs ``as good as'' as a similar classifier trained on the original Bag-of-$L$-cooccurrences vectors. This was the first result that provided provable quantification of the power of any text embedding.\footnote{\citet{arora2018compressed} have additional results on the powerfulness of low-memory LSTM (Long Short Term Memory) network embeddings, by showing that the LSTMs under certain initialization can simulate these DisC embeddings. Our more general RIP results can be easily used to generalize these results too (details omitted).}
They achieve this result by connecting this problem with the theory of compressed sensing, an idea that we build upon here. Let $V \in \R^{p \times d}$ be a matrix whose columns are the embeddings for all the words in $\mathcal{W}$.~\citet{arora2018compressed} result relies on establishing the restricted isometry property on the matrix $V^{(\ell)} \in \R^{p \times \binom{d}{\ell}}$ where $V$ is an i.i.d.\ $\pm 1$-random matrix (i.e., random vectors are used as word embeddings).

\noindent Linear transformations are regularly used to transfer between different embeddings or to adapt to a new domain~\citep{bollegala2017learning,arora2018linear}. The linear transformation can encode {\em contextual information}, an idea utilized recently by~\citep{khodak2018carte} who applied a linear transformation on the DisC embedding scheme to construct a new embedding scheme (referred to as {\em \`a la carte} embedding), and empirically showed that it outperforms many other popular word sequence embedding schemes. Now akin to~\citep{arora2018compressed} result on DisC embeddings, our results shed some theoretical insights into the performance on linear transformations of DisC embeddings. In particular, our RIP results on $X V^{(\ell)}$, where is $V \in \R^{p \times d}$ is an i.i.d.\ centered bounded random matrix, provides provable performance guarantees on linear transformation (defined by $X$) of DisC embeddings in a linear classification setting, under a stable rank assumption on $X$ (see Corollary~\ref{cor:emb}). We expand on this application in Section~\ref{sec:word}.

\item\textbf{Deep Linear Networks.} A deep linear network $\mathcal{D}\,:\, \R^d \rightarrow \R^n$ is a neural network that has multiple hidden layers but have no nonlinearities between layers (typically $n \ll d$). That is, for a given datapoint $\x$, the output $\mathcal{D}(\x)$ is computed via a series $\mathcal{D}(\x) = W_t W_{t-1} \dots W_1 \x$ of matrix multiplications. Such models are commonly used as a tool for theoretically understanding deep neural networks~\citep{arora2018convergence,laurent2018deep}.  A simple question that arises is whether low-dimensional output of a deep linear network  can be used as a surrogate representation for classification, i.e., what happens if instead of $(\x_1,\upsilon_1),\dots,(\x_b,\upsilon_b)$ where $\upsilon_i$ is the label on $\x_i$, $(\mathcal{D}(\x_1),\upsilon_1),\dots,(\mathcal{D}(\x_b),\upsilon_b)$ is used for classification? Our RIP results help answer this question. For example, if $W_1$ is a matrix from one of the above random matrix families\footnote{Randomness is commonly used in deep networks for initializing weights~\citep{arora2018convergence} before training, and sometimes random weights are used without training~\citep{saxe2011random}.}, and the matrix $X = W_t W_{t-1} \dots W_2$ satisfies the stable rank condition, then we get that the matrix $W_t W_{t-1} \dots W_1$ satisfies RIP with high probability. Then the known results in compressed learning (see for example Theorem~\ref{thm:arloss}) can be used to relate the performance of a linear classifier trained on the compressed representation $(\mathcal{D}(\x_1),\upsilon_1),\dots,(\mathcal{D}(\x_b),\upsilon_b)$ to a similar classifier trained on the original representation $(\x_1,\upsilon_1),\dots,(\x_b,\upsilon_b)$. 

\item\textbf{Linear Transformation and Johnson-Lindenstrauss Embedding.} {\em Johnson-Lindenstrauss (JL) embedding lemma} states that any set of $m$ points in high dimensional Euclidean space $\R^d$ can be embedded into $p = O(\log(m)/\epsilon^2)$ dimensions, without distorting the distance between any two points by more than a factor between $1-\epsilon$ and $1+\epsilon$~\citep{johnson1984extensions}. JL lemma has become a valuable tool for dimensionality reduction~\citep{TCS-060}. Typically, the embedding is constructed through a random linear map referred to as an $\epsilon$-JL matrix. Let $R \in \R^{p \times d}$ be an $\eps$-JL matrix.
Consider a set of $m$ points $\a_1,\dots,\a_m \in \R^d$, and let $R\a_1,\dots,R \a_m$ be their low-dimensional embedding. A natural question to ask is: under what fixed linear transformations of $R \a_1,\dots,R \a_m$ does this distance preservation property still hold? More concretely, let $X \in \R^{n \times p}$ be a linear transformation. Does the JL distance preservation property still hold for $XR \a_1,\dots,XR \a_m$? Our results answers this question because of the close connection between JL-embedding and the restricted isometry property.~\citet{krahmer2011new} showed that, when the columns of a matrix satisfying RIP are multiplied by independent random signs, any $(O(\log m),O(\epsilon))$-RIP matrix becomes an $\epsilon$-JL matrix for a fixed set of $m$ vectors with probability at least $1- m^{-\Omega(1)}$.\!\footnote{The connection in other direction going from JL-embedding to RIP is also well-known~\citep{baraniuk2008simple}.} This means, one can now use our RIP results discussed above. For example, if the entries of $R$ are drawn i.i.d.\ from a centered symmetric subgaussian distribution, then for a $X$ (scaled to have  unit Frobenius norm), if $\sr(X) = \tilde{\Omega}(\log(m)/\epsilon^2)$\footnote{The $\tilde{\Omega}$ notation hides polylog factors in $d,1/\eps$.}, we have that $XR \a_1,\dots,XR \a_m$ satisfies the JL-embedding property with high probability (as multiplying with random signs does not change the distribution $R$).

\item\textbf{Compressed Sensing under Linear Transformations.} Compressed sensing algorithms are designed to recover approximately sparse signals, and a popular sufficient condition for a matrix to succeed for the purposes of compressed sensing is given by the restricted isometry property (see, e.g., Theorem~\ref{thm:candes}). 
Another question that can be addressed using our results is how does the recovery guarantee hold if we apply a linear transformation on the compressed signals. 
Our RIP results establishes conditions on $X$, for the class of random matrices $R$ described earlier, under which given $\y = X R \x + \e$ where $\e$ is a noise vector, one can recover an approximately sparse vector $\x$.

\item\textbf{Source Separation.} Separation of underdetermined mixtures is an important problem in signal processing. Sparsity has often been exploited for source separation. The general goal with source separation is that given a signal matrix $S \in \R^{n_1 \times c}$,  a mixing matrix $M \in \R^{n_1 \times a}$, and a dictionary $\Phi \in \R^{b \times c}$,  is to find a matrix of coefficients $C \in \R^{a \times b}$ such that $S \approx MC\Phi$ and $C$ is as sparse as possible.  The dictionary is generally overcomplete, i.e., $b > c$. The connection between source separation and compressed sensing was first noted by~\citep{blumensath2007compressed}.\footnote{In a variant of this problem, called the {\em blind source separation}, even the mixing matrix $M$ is assumed to be unknown.}  It is easy to recast (see Appendix~\ref{app:SS} for details) the source separation problem as a compressed sensing problem of the form: $\s = XR\c$, where the goal is to estimate the sparse $\c$ (entries of the matrix $C$). Here, $\s \in \R^n$, $X \in \R^{n \times p}$, $R \in \R^{p \times d}$, and $\c \in \R^d$, with $n = n_1c$, $d = ab$, $p = ac$. 
Hence, our RIP results establishes conditions on the mixing matrix $M$, for the classes of random matrices (dictionaries) described earlier, under which we can get recovery guarantees on $\c$ (entries of the matrix $C$).



\item\textbf{Singular Values of Correlated Random Matrices.}  Next application is a simple consequence of the restricted isometry property.
Understanding the singular values of random matrices is an important problem, with lots of applications in machine learning and in the field of non-asymptotic theory of random matrices~\citep{vershynin2016high}. 
Restricted isometry property of $M=XR$ can be interpreted in terms of the extreme singular values of submatrices of $M$. Indeed, the restricted isometry property (assuming for simplicity that $X$ has unit Frobenius norm\footnote{Otherwise results can be appropriately scaled.}) equivalently states that the inequality 
\begin{align} \label{eqn:boundsall}
\sqrt{1-\epsilon} \leq \lambda_{\min}(M_I^\top M_I) \leq \lambda_{\max}(M_I^\top M_I) \leq \sqrt{1+\epsilon}.
\end{align} 
holds for all $n \times k$ submatrices $M_I= XR_I$, those formed by the columns of $M$ indexed by sets $I$ of size $k$.  If the columns of $R$ are independent random vectors drawn from distribution $\mathcal{R}$ over $\R^p$, we can think of our restricted isometry results as bounds on the singular values of matrices formed from by linear transformation of random vectors (or equivalently, the eigenvalues of the sample covariance matrix drawn from the distribution $X\mathcal{R}$). 
For example, let $Q \in \R^{p \times k}$ be a matrix whose columns are independent random vectors from a subgaussian distribution or satisfying convex concentration property, then our RIP results through~\eqref{eqn:boundsall} bounds {\em all} the singular values on $XQ$ (with high probability) under the condition that $\sr(X) =\Omega(k\log(1/\epsilon)/\epsilon^2)$. 
Previously only a spectral norm bound on $XQ$ for the cases where $Q$ has independent entries drawn from a subgaussian or heavy-tailed distributions was known~\citep{RVHanson-Wright,vershynin2011spectral}.
\end{list}

\section{Preliminaries} \label{sec:prelim}
\noindent\textbf{Notation.} We use $[n]$ to denote the set $\{1 \etc n\}$. We denote the Hadamard (elementwise) product by $\odot$. For a set $I \subseteq [d]$, $I^{\co}$ denotes its complement set.

Vectors used in the paper are by default column vectors and are denoted by boldface letters. For a vector $\v$, $\v^\top$ denotes its transpose, $\|\v\|$ and $\|\v\|_1$ denote the $L_2$ and $L_1$- norm respectively, and $\supp(\v)$ its support. $\e_i \in \R^d$ denote the standard basis with $i$th coordinate $1$. For a matrix $M$, $\|M\|$ denotes its spectral norm, $\|M\|_F$ denotes its Frobenius norm, and $M_{ij}$ denotes its $(i,j)$th entry, $\diag(M)$ denotes the diagonal of $M$, and $\offdiag(M)$ denotes the off-diagonal of $M$. $\mathbb{I}_n$ represents the identity matrix in dimension $n$.  For a vector $\x$ and set of indices $S$, let $\x_S$ be the vector formed by the entries in $\x$ whose indices are in $S$, and similarly, $M_S$ is the matrix formed by columns of $M$ whose indices are in $S$. 
 
The Euclidean sphere in $\R^d$ centered at origin is denoted by $S^{d-1}$. We call a vector $\a \in \R^d$, {\em $k$-sparse}, if it has at most $k$ non-zero entries. Denote by $\Sigma_k$ the set of all vectors $\a \in S^{d-1}$ with support size at most $k$: $\Sigma_k=\{ \a \in S^{d-1} \,: \, | \supp(\a) | \leq k\}$. 
Stable rank (denoted by $\sr()$) of a matrix $M$ is defined as: 
$$\sr(M) = \norm{M}_\HS^2 / \norm{M}^2.$$
Stable rank cannot exceed the usual rank. The stable rank is a more robust notion than the usual rank because it is largely unaffected by tiny singular values.  

Throughout this paper $C,c,C'$, also with subscripts, denote positive absolute constants, whose value may change from line to line. In Appendix~\ref{app:prelim} we discuss additional preliminaries about subgaussian/subexponential random variables,  sparse recovery, and Hanson-Wright inequality.

\noindent\textbf{Restricted Isometry.}~\citet{candes2005decoding} introduced the following isometry condition on matrices $M$. It is perhaps one of the most popular property of a matrix which makes it ``good'' for compressed sensing. 
\begin{definition} \label{def:RIP}
Let $M$ be an $n \times d$ matrix with real entries. Let $\epsilon \in (0,1)$ and let $k < d$ be an integer. We say that $M$ satisfies $(k,\epsilon)$-RIP if for every $k$-sparse vector $\u \in S^{d-1}$ (i.e., $\u \in \Sigma_k$)
$$ (1-\epsilon)  \leq \| M \u \|  \leq (1+\epsilon).$$
\end{definition}
Thus, $M$ acts almost as an isometry when we restrict attention to $k$-sparse vectors. 
Geometrically, the restricted isometry property guarantees that the geometry of $k$-sparse vectors $\u$ is well preserved by the measurement matrix $M$. In turns out that in this case, given a (noisy) compressed measurement $M\x$, for approximately sparse $\x$, one can recover $\x$ using a simple convex program. Theorem~\ref{thm:candes} provides a bound on the worst-case recovery performance for uniformly bounded noise. Similar recovery results using RIP on the measurement matrix $M$ are also well-known under other interesting settings  (we refer the reader to the survey on this topic in~\citep{eldar2012compressed}).

We investigate the restricted isometry property of matrices $M \in \R^{n \times d}$ that can be factored as $M = XR$, where $R$ will be some random matrix and $X$ is a fixed matrix. Now RIP is not invariant under scaling, i.e., given a RIP matrix $M$, changing $M$ to some $a M$ for some $a \in \R$ scales both the left and right hand side of the inequality in the Definition~\ref{def:RIP} by $a$.  So while working in the $XR$-model we have to adjust for the scaling introduced by $X$, and we use a generalization of  Definition~\ref{def:RIP} appropriate for this setting. 
\begin{definition} [RIP Condition] \label{defn:RIPmod}
A matrix $M = XR$ satisfies $(k,\epsilon)$-RIP if for every $k$-sparse vector $\u \in S^{d-1}$ (i.e., $\u \in \Sigma_k$)
$$ (1-\epsilon) \|X\|_F  \leq \| XR \u \|  \leq (1+\epsilon) \| X \|_F.$$
\end{definition}
This scaling by $\|X\|_F$ is unavoidable, because even if $R$ is a matrix with centered uncorrelated entries of unit variance, then $\E[\|XR\u\| ^2]=\|X\|_F^2$ for any $\u$ with $\|\u\|=1$.  Note that it is easy to reconstruct standard RIP scenarios by choosing $X$ appropriately. For example, if $R$ contains i.i.d.\ subgaussian random variables with zero mean and unit variance, then setting $X = \mathbb{I}_n/\sqrt{n}$, leads to standard setting of $XR= R/\sqrt{n}$ and $\|X\|_F=1$. 

\noindent\textbf{Hanson-Wright Inequality.} An important concentration tool used in this paper is the Hanson-Wright inequality (Theorem~\ref{thm: HW}) that investigates concentration of a quadratic form of independent centered subgaussian random variables, and its recent extensions~\citep{zhou2015sparse,adamczak2015note}. A slightly weaker version of this inequality was first proved in~\citep{hanson1971bound}. 

\begin{theorem}[Hanson-Wright Inequality~\citep{RVHanson-Wright}] \label{thm: HW}
Let $\x = (x_1,\ldots,x_n) \in \R^n$ be a random vector with independent components $x_i$ which satisfy $\E[x_i] = 0$ and $\|x_i\|_{\psi_2} \leq K$. Let $M$ be an $n \times n$ matrix. Then  for every $t \ge 0$,
\[ \Prob{ \left |\x^\top M \x - \E[\x^\top M \x] \right | > t} \le 2 \exp \Big( - c \min \Big \{ \frac{t^2}{K^4 \norm{M}_\HS^2}, \frac{t}{K^2 \norm{M}} \Big \} \Big ). \]
\end{theorem}

\section{Restricted Isometry of $XR$ with ``Random'' $R$} \label{sec:randR}
In this section, we investigate the restricted isometry property for the class of $XR$ matrices, for various classes of random matrices $R \in \R^{p \times d}$. Missing details from this section are collected in Appendix~\ref{app:randR}.

As a warmup, we start with the simplest case where $R$ is a centered i.i.d.\ subgaussian random matrix, and build on this result, where we consider various other general families of $R$ such as those constructed using low-randomness, with sparsity structure, or satisfying a convex concentration condition. 

Theorem~\ref{thm:rand} presents the result in case $R$ is an i.i.d.\ subgaussian random matrix. The proof idea here is quite simple, but provides a framework that will be helpful later. Under the stable rank condition, a net argument, along with Hanson-Wright inequality implies, $\Pr[\exists I \subset [d], \ |I|=k, \ \ \|XR_I\| \ge C_0 \eps \|X\|_{F}] \le \exp (-c_0 \eps^2 \sr(X)/K^4)$, where $R_I$ is the $p \times k$ submatrix of $R$ with columns from the set $I$. Also, by the same Hanson-Wright, for any $\u \in S^{d-1}$, $\Pr[|\| XR \u \| - \| X \|_F | \ge \eps \|X\|_{F}] \le \exp (-c_1 \eps^2 \sr(X)/K^4)$. Using a bound on the net size for sparse vectors gives all the required ingredients for the following theorem.

\begin{theorem} \label{thm:rand}
Let $X$ be an $n \times p$ matrix. Let $R=(R_{ij})$ be a $p \times d$ matrix whose entries are with independent entries such that $\E[R_{ij}] = 0$, $\E[R_{ij}^2] = 1$, and $\|R_{ij}\|_{\psi_2} \leq K$. Let $\epsilon \in (0,1)$, and let $k \in \N$ be a number satisfying $\sr(X)\ge CK^4 \frac{k}{ \epsilon^2} \log \left ( \frac{d}{k} \right )$. Then with probability at least $1-\exp(-c \epsilon^2 \sr(X)/K^4)$, the matrix $XR$ satisfies $(k, \epsilon)$-RIP, i.e.,
 \[\forall \u \in \Sigma_k, \; (1-\epsilon) \norm{X}_{\HS} \le \|XR\u\| \le (1+\epsilon) \norm{X}_{\HS}.\]
\end{theorem}
Note that under the assumption the stable rank, the probability $1-\exp(-c \epsilon^2 \sr(X)/K^4)$ is at least $1-\exp(-c' k)$.

\begin{remark}
Notice that there is no direct condition on $n$, except that comes through the stable rank assumption on $X$, as $\sr(X) \leq \rank(X) \leq \min\{n,p\}$. Indeed one should expect this to happen. For example,  if we take a matrix $X$ and add a bunch of zero rows to it this would increase $n$, but should not change the recovery properties. This suggests that we need a notion of a ``true'' dimension of the range of $X$, and the stable rank is the one. 
\end{remark}

This assumption on the stable rank is optimal up to constant factors. For example, when $X=\mathbb{I}_n$ is the identity matrix and $R$ is a standard Gaussian matrix, then $\sr(X)=\rank(X)=n$, and therefore the stable rank condition just becomes $n = \Omega((k/\eps^2) \log(d/k))$. We know that up to constant, this dependence of $n$ on $d$ and in $k$ is optimal~\citep{foucart2010gelfand}. This shows that the lower bound on the stable rank in Theorem~\ref{thm:rand} cannot be improved in general.

In the following, we investigate other popular families of random matrices, extending the result in Theorem~\ref{thm:rand} in various directions. 
\begin{list}{{\bf (\alph{enumi})}}{\usecounter{enumi}
\setlength{\leftmargin}{8pt}
\setlength{\listparindent}{0pt}
\setlength{\parsep}{0pt}}
\item{\textbf{Independent Subgaussian Columns:}} Consider the case where the columns of the matrix $R$ are independently drawn isotropic subgaussian random vectors. In this case the proof proceeds as in Theorem~\ref{thm:rand}, by applying in this case an extension of Hanson-Wright inequality (Theorem~\ref{thm: HW}) to subgaussian random vectors~\citep{vershynin2016high}, and by noting that for any fixed vector $\u \in S^{d-1}$, $R\u$ will be an isotropic subgaussian random vector as it is a linear combination of columns of $R$ which are all independent.  
\item{\textbf{Low Randomness:}} Optimal use of randomness is an important consideration when designing the matrix $R$. In the dimensionality reduction literature much attention has been given in obtaining explicit constructions of $R$ minimizing the number of random bits used (see, e.g.,~\citep{kane2011almost} and references therein). For a fixed $X$, we show that as long as $R$ satisfies $2 \sr(X)$-wise independence, then $XR$ satisfies RIP under the same (up to constant) stable rank condition on $X$ as in Theorem~\ref{thm:rand}. So in effect, one can reduce the number of random bits from $O(pd)$ (in Theorem~\ref{thm:rand}) to $O(\sr(X) \log (pd))$. 

Again we start as in Theorem~\ref{thm:rand} with the Hanson-Wright inequality, but in this case the proof has to deal with the lack of independence in $R$. Our general strategy in this proof will be to rely on higher moments where we can treat certain variables as independent. The main technical step is to derive a concentration bound for $\| X R \u\|$ for a fixed $\u \in S^{d-1}$, which in this case is obtained through analyzing certain higher moments of $\|XR \u \|^2 - \E[\|XR \u \|^2 ]$ using moment generating functions. For $2\sr(X)$-wise independent $R$, we establish $(k,\eps)$-RIP on $XR$ if $\sr(X) = \Omega((k/\eps^2) \log(d/k))$. See Theorem~\ref{thm:lownoise} for a formal statement. 

\item{\textbf{Sparse-Subgaussian:}} Sparsity is a desirable property in the measurement matrix because it leads to faster computation matrix-vector product.  For example, if $R$ is drawn from a distribution over matrices having at most $s$ non-zeroes per column, then $R \x$ can be computed in time $s  \| \x \|_0$. We use the sparse-subgaussian model of random matrices. 

Our result here (Theorem~\ref{thm:sparse}) utilizes the recently introduced sparse Hanson-Wright inequality~\citep{zhou2015sparse}. Similar to Hanson-Wright inequality, sparse Hanson-Wright inequality provides a large deviation bound for a quadratic form. However, in this case, the quadratic form is sparse, and is of the form $(\x \odot \zeta)^\top M (\x \odot \zeta)$ where $\x$ is an random vector with independent centered subgaussian components ($\psi_2$-norm bounded by $K$) and $\zeta$ contains independent Bernoulli($\beta$) random variables. We use the sparse Hanson-Wright inequality to obtain the necessary subgaussian concentration bounds (Lemmas~\ref{lem:sparsefixed},~\ref{lem:sparsenrom}) in a proof framework similar to Theorem~\ref{thm:rand}. For a Bernoulli-subguassian matrix $R$, we establish $(k,\eps)$-RIP on $XR$ if $\sr(X) = \Omega((K^4 k/(\beta \eps^2)) \log(d/k))$. See Theorem~\ref{thm:sparse} for a formal statement.

\item{\textbf{Under Convex Concentration:}}  Convex concentration property is a generalization of standard concentration property (such as Gaussian concentration) by requiring concentration to hold only for 1-Lipschitz convex functions.
\begin{definition} [Convex Concentration Property] \label{defn:ccp}
Let $\x$ be a random vector in $\R^d$. We will say that $\x$ has the convex concentration property (c.c.p) with constant $K$ if for every $1$-Lipschitz convex function $\phi \,:\, \R^d \rightarrow \R$, we have $\E[|\phi(\x)|] < \infty$ and for every $t > 0$,
$$\Pr[|\phi(\x) - \E[\phi(\x)]| \geq t]  \leq 2 \exp(-t^2/K^2).$$
\end{definition}
\noindent The class of distributions satisfying c.c.p is extremely broad~\citep{ledoux2001concentration}. Some examples include: \begin{inparaenum}[i)]  \item Gaussian random vectors drawn from $N(0,\Sigma)$ have c.c.p\ with $K^2 = 2\| \Sigma \|$, \item random vectors that are uniformly distributed on the sphere $\sqrt{d} S^{d-1}$ have c.c.p\ with constant $K = 2$, \item subclass of logarithmically concave random vectors, \item random vectors with possibly dependent entries which satisfy a {\em Dobrushin} type condition, and \item random vectors satisfying the {\em logarithmic Sobolev} inequality. \end{inparaenum}  

The main technical part here is to show that a linear combination of independent vectors having a convex concentration property would have this property as well (Lemma~\ref{lem:ccp}). To this end, we start with a fixed $\u=(u_1,\dots,u_d) \in S^{d-1}$ and construct a martingale with variables $\E [ \phi ( \sum_{i=1}^d \r_i u_i  ) \,|\, \r_1,\dots,\r_j  ]$ where $\phi \, : \, \R^p \rightarrow \R$ is a 1-Lipschitz  convex function and $\r_i$ is the $i$th column in $R$, and then apply  Azuma's inequality to it. Once we have this established, verifying RIP consists of applying the Hanson-Wright inequality of~\citep{adamczak2015note} for isotropic vectors having convex concentration property in a proof framework similar to Theorem~\ref{thm:rand}. For a matrix $R$ with independent columns satisfying c.c.p with constant $K$, we establish $(k,\eps)$-RIP on $XR$ if $\sr(X) = \Omega((K^4 k/\eps^2) \log(d/k))$. See Theorem~\ref{thm:ccp} for a formal statement. 

\end{list}

\section{Restricted Isometry under $\ell$-way Column Hadamard-product} \label{sec:had}
In this section, we investigate the restricted isometry property for a class of correlated random matrices motivated by theoretically understanding the effectiveness of word vector embeddings. Missing details from this section are collected in Appendix~\ref{app:had}.  

To introduce this setting, let us start with the definition of a matrix product operation introduced by~\citep{arora2018compressed} to construct their distributed cooccurrence (DisC) word embeddings. 

\begin{definition} [$\ell$-way Column Hadamard-product Operation] \label{defn:3}
Let $M$ be an $p \times d$ matrix, and let $\ell \in \N$. The $\ell$-way column Hadamard-product operation constructs a $p \times \binom{d}{\ell}$ matrix $M^{(\ell)}$ whose columns indexed by a sequence $1\leq i_1 < i_2 \dots < i_\ell \leq d$ is the elementwise product of the $i_1, i_2, \dots i_\ell$-th columns of $M$, i.e., $(i_1,i_2,\dots,i_\ell)$-th column in $M^{(\ell)}$ is $\m_{i_1} \odot \dots \odot \m_{i_\ell}$, where $\m_j$ for $j \in [d]$ is the $j$th column in $M$.
\end{definition}

\citet{arora2018compressed}, based on an application of a result of~\citep{foucart2017mathematical} on RIP for {\em bounded orthonormal systems}, showed that if $R \in \R^{n \times d}$ is an i.i.d.\ random sign matrix, and if $n = \Omega((k/\eps^2) \log(d^\ell/\gamma))$,  then with probability at least $1-\gamma$, $R^{(\ell)}$ satisfies restricted isometry property. In this paper, we investigate RIP on $X R^{(\ell)}$. It is not hard to see that achieving RIP for $R^{(\ell)}$ is {\em simpler} than achieving RIP for $XR^{(\ell)}$. Intuitively, in the $R^{(\ell)}$ case one has only to ensure that $\|R^{(\ell)} \u \|$ is close to  a constant with high probability for a fixed sparse unit vector $\u$. In the case of $XR^{(\ell)}$, in addition to it, one has to guarantee that the direction of the vector $R^{(\ell)} \u$ is more or less uniformly distributed over the sphere. 

For $\ell = 1$, Theorem~\ref{thm:rand} holds, and we get RIP under the condition $\sr(X) = \Omega((k/\eps^2) \log(d/k))$. 
Below, we first analyze the case of $\ell=2$, and then for larger $\ell$'s. In the motivating application discussed in Section~\ref{sec:word}, $\ell$  is generally a small constant.

\smallskip
\noindent\textbf{Analysis for $\ell=2$.} Let $R$ be a random matrix with centered $\tau$-bounded entries.\!\footnote{The centered assumption is necessary because to hope for RIP with high probability, we must have it in average. The boundedness assumption is an artifact of our proof approach. The product of subgaussians is not subgaussian. However, the class of bounded random variables is closed under product.} We prove that the matrix $XR^{(2)}$ satisfies RIP with high probability provided that $\sr(X)$ is sufficiently large. Note that we do not require the entries in $R$ to be identically distributed.
\begin{theorem}  \label{thm: RIP}
Let $X$ be an $n \times p$ matrix, and let $R$ be a $p \times d$ random matrix with independent entries $R_{ij}$ such that $\E[R_{ij}]=0, \E[R_{ij}^2]=1, \text{and } |R_{ij}| \le \tau \text{ almost surely}$.
Let $\eps \in (0,1)$, and let $k \in \NNN$ be a number satisfying $\sr(X) \ge \frac{C \tau^8 k^2}{\eps^2}  \log \left( \frac{d^2}{k} \right)$.  Then with probability at least $1- \exp(-c \eps^2 \sr(X)/(k\tau^8))$, the matrix $XR^{(2)}$ satisfies the $(k,\eps)$-RIP property, i.e.,  for any $\u \in S^{\binom{d}{2}-1}$ with $|\supp(\u)| \le k$,
\[(1-\eps) \norm{X}_{\HS} \le \norm{XR^{(2)}\u} \le  (1+\eps) \norm{X}_{\HS}.\]
\end{theorem}
\begin{proof}
Let $\u \in \R^{\binom{d}{2}}$ be a vector with $|\supp(\u)| \le k$. Let $u_{ij}$ be the $(i,j)$th element in $\u$ with $(i,j) \in \binom{[d]}{2}$.
The random variable $\norm{XR^{(2)}\u}^2$ is order 4 homogenous chaos in terms of the random variables $R_{ij}$. Establishing concentration for this chaos can be a difficult task, so we will approach the problem from a different angle.
Let $l \in [p]$, and define 
\[y_l=\sum_{(i,j) \in \binom{[d]}{2} } R_{li} R_{lj} u_{ij}.\]
Note that the random variables $y_l, \ l \in [p]$ are independent.
We will estimate the $\psi_2$-norm of $y_l$ and use the Hanson-Wright inequality (Theorem~\ref{thm: HW}) to establish the concentration for the norm of $XR^{(2)}\u = X \y$ (where $\y=(y_1,\dots,y_p)$). Note that the support of $\u$ contains the pair $(i,j)$ with $u_{ij} \neq 0$. 
Directly from the triangle inequality, one can get 
\begin{align} \label{eqn:triangle}
\norm{\sum_{(i,j) \in \binom{[d]}{2} }  R_{li} R_{lj} u_{ij} }_{\psi_2} \le \sum_{(i,j) \in \supp(\u)}  \norm{R_{li} R_{lj} u_{ij} }_{\psi_2} = O(\tau^2 \norm{\u}_1) = O(\tau^2 \sqrt{k} \norm{\u}),
\end{align}
since $|\supp(\u)| \le k$. This estimate is, however, too wasteful. We will prove a more precise one using a special decomposition of the vector $\u$. It is based on a novel induction procedure. For simplicity, we ignore the subscript $l$, and denote $R_{li}$ by $r_i$ and $R_{lj}$ by $r_j$. We explain the induction idea here, deferring the entire proof to Appendix~\ref{app:had}.

The support of vector $\u$ can be viewed as an $I \times I^{\co}$ matrix with at most $k$ non-zero entries. If each row of this matrix contains at most one non-zero entry, we can condition on $r_j, j \in I^{\co}$ and get a bound on the $\psi_2$-norm which does not depend on $k$ using Hoeffding's inequality. The same is true if each column of the matrix contains at most one non-zero entry. This suggests that intuitively, the worst case scenario occurs when the non-zero entries of $\u$ form a $\sqrt{k} \times \sqrt{k}$ submatrix. This submatrix can be split into the sum of  $\sqrt{k}$ rows, which in combination with Cauchy-Schwarz inequality allows to bound the $\psi_2$-norm by $k^{1/4}$. But it is not clear how to extend this splitting to the case of $\u$ with an arbitrary support. However, the $\sqrt{k} \times \sqrt{k}$ matrix can be split in a slightly different way. Namely, we separate the first row of the $\sqrt{k} \times \sqrt{k}$ matrix of support of $\u$, so that the remaining part is a $(\sqrt{k}-1) \times \sqrt{k}$ matrix. After that, we separate the first column of the remaining matrix which leaves a $(\sqrt{k}-1) \times (\sqrt{k}-1)$ matrix. Alternating between rows and columns, we get a $(\sqrt{k}-t) \times (\sqrt{k}-t)$ matrix after $2t$ steps. This process ends in $2 \sqrt{k}$ steps, which yields the same $k^{1/4}$ bound for the $\psi_2$-norm. We show below that this method can be extended to a vector $\u$ with an arbitrary support.

\begin{lemma} \label{lem: psi_2}
Let $\u \in \R^{\binom{d}{2}}$ be a vector with $|\supp(\u)| \le k$. Let $r_i, \ i \in [d]$ be independent random variables with $\E[r_{i}]=0, \E[r_{i}^2]=1, \text{and } |r_{i}| \le \tau \text{ almost surely}$. Then
\[ \norm{\sum_{(i,j) \in \binom{[d]}{2} }  r_i r_j u_{ij} }_{\psi_2} \le 16 \tau^2 k^{1/4} \norm{\u}.\]
\end{lemma}
From Lemma~\ref{lem: psi_2}, we finish the proof of the theorem by combining the Hanson-Wright inequality (Theorem~\ref{thm: HW}) and a net argument. 

Recall that for any $k$-sparse vector $\u \in \R^{\binom{d}{2}}$,  $XR^{(2)}\u=X\y$, where $\y = (y_1,\dots,y_{p})$ is a random vector with independent coordinates such that for all $l \in [p]$
\[ \E[y_l]=0, \quad \E[y_l^2]= \norm{\u}^2, \quad \text{and } \norm{y_l}_{\psi_2} \le C \tau^2 k^{1/4} \norm{\u} \;\text{(from Lemma~\ref{lem: psi_2})}.\]
By the volumetric estimate we can choose an $(1/2C_2)$-net $\NN$ in the set of all $k$-sparse vectors in $S^{\binom{d}{2}-1}$ such that 
\[ |\NN| \le \binom{\binom{d}{2}}{k} \left(6 C_2 \right)^k \le \exp \left(k \log \left ( \frac{C_0 d^2}{k} \right ) \right ).\]
Using Corollary~\ref{cor:subgauss}, for any $\u \in S^{\binom{d}{2}-1}$ with $|\supp(\u)| \le k$ (and $\y = R^{(2)} \u$), 
\begin{align*}
 \Pr \left[ \left| \norm{X\y} - \norm{X}_{\HS} \right| > \eps \norm{X}_{\HS} \right]
 &\le 2 \exp \left( - \frac{C \eps^2}{\max_l \norm{\y_l}_{\psi_2}^4} \sr(X) \right) \le 2 \exp \left( - \frac{C_1 \eps^2}{\tau^8 k} \sr(X) \right).
\end{align*}
Combining this with the union bound over $\u \in \NN$ and using the assumption on $\sr(X)$, we get
\begin{align} \label{eqn:1a}
 \Pr \left[ \exists \u \in \NN, \  \left| \norm{XR^{(2)}\u} - \norm{X}_{\HS} \right| > \eps \norm{X}_{\HS} \right]
 &\le \exp \left(k \log \left ( \frac{C_0 d^2}{k} \right ) \right ) \cdot 2 \exp \left( - \frac{C_1 \eps^2}{\tau^8 k} \sr(X) \right).
\end{align}
Similar from Hanson-Wright (Corollary~\ref{cor: product norm}), we can derive that
\begin{align} \label{eqn:1b}\Prob{ \exists I \in \binom{[d]}{2}, \ |I| = k, \ \| X R^{(2)}_I\| > C_1\eps \norm{X}_\HS} \le \exp \left( - \frac{c_1 \eps^2}{\tau^8 k} \sr(X) \right).\end{align}
Now using~\eqref{eqn:1a} and~\eqref{eqn:1b} and an approximation idea (as in Theorem~\ref{thm:rand}) completes the proof of Theorem~\ref{thm: RIP}.
\end{proof}


\noindent\textbf{Analysis for $\ell \geq 3$.} While extending the stronger $\psi_2$-norm estimate given by Lemma~\ref{lem: psi_2} to these larger $\ell$'s seems tricky, the looser bound obtained through a triangle inequality argument (as in~\eqref{eqn:triangle}) still holds. This leads to the following theorem.
\begin{theorem}  \label{thm: RIP2}
Let $X$ be an $n \times p$ matrix, and let $R$ be a $p \times d$ random matrix with independent entries $R_{ij}$ such that $\E[R_{ij}]=0, \E[R_{ij}^2]=1, \text{and } |R_{ij}| \le \tau \text{ almost surely}$.
Let $\ell \geq 3$ be a constant. Let $\eps \in (0,1)$, and let $k \in \NNN$ be a number satisfying $\sr(X) \ge \frac{C \tau^{4\ell} k^3}{\eps^2} \log \left( \frac{d^\ell}{k} \right)$.  Then with probability at least $1- \exp(-c \eps^2 \sr(X)/(k^2\tau^{4\ell}))$, the matrix $XR^{(\ell)}$ satisfies the $(k,\eps)$-RIP property, i.e.,  for any $\u \in S^{\binom{d}{\ell}-1}$ with $|\supp(\u)| \le k$,
\[(1-\eps) \norm{X}_{\HS} \le \norm{XR^{(\ell)}\u} \le  (1+\eps) \norm{X}_{\HS}.\]
\end{theorem}
\begin{remark}
It is tempting to conjecture  that the dependence on $k$ in the $\sr(X)$ condition in both Theorems~\ref{thm: RIP} and~\ref{thm: RIP2} should be linear, but as mentioned above analyzing these matrix families appears challenging, and it is plausible that a superlinear dependence on $k$ might in fact be unavoidable. 
\end{remark}
\subsection{Application to Understanding Effectiveness of Word Embeddings} \label{sec:word}
Word embeddings which represent the ``meaning'' of each word via a low-dimensional vector, have been widely utilized for many natural language processing and machine learning applications. We refer the reader to~\citep{mikolov2013distributed,pennington2014glove} for additional background about word vector embeddings. Individual word embeddings can be extended to embed word sequences (such as a phrase or sentence) in multiple ways. There has been a recent effort to better understand the effectiveness of these embeddings, in terms of the information they encode and how this relates to performance on downstream tasks~\citep{arora2016simple,arora2018compressed,khodak2018carte}. In this paper, we work with a recently introduced word sequence embedding scheme, called {\em distributed cooccurrence} (DisC) embedding, that has been shown to be empirically effective for downstream classification tasks and also supports some theoretical justification~\citep{arora2018compressed}. 

We in fact investigate linear transformations of these DisC embedding vectors (Definition~\ref{defn:4}). Linear transformations are commonly applied over existing embeddings to construct new embeddings in the context of {\em domain adaptation}, {\em transfer learning}, etc. For example, recently~\citep{khodak2018carte} applied a (learnt) linear transformation on the DisC embedding vectors to construct a new embedding scheme, referred to as {\em \`a la carte} embedding, which they empirically show regularly outperforms the DisC embedding. 
Our results show, under some conditions, these linearly transformed  DisC embeddings have provable performance guarantees for linear classification.
Before stating our result formally, we need some definitions.

Let $\mathcal{W}$ denote the vocabulary set (collection of some words) with $|\mathcal{W}|=d$. We assume each word $\w \in \mathcal{W}$ has a vector representation $\v_{w} \in \R^p$. Let $V \in \R^{p \times d}$ denote the matrix whose columns are these word embeddings. A $\ell$-gram is a {\em contiguous} sequence of $\ell$ words in a longer word sequence. Here, $\ell=2$ is referred commonly to as bigram, $\ell=3$ is referred to as trigram, etc.\!\footnote{For example, if the word sequence equals (``the'', ``cow'', ``jumps'' ,``over'', ``the'', ``moon''), then it has 5 different words. The collection of $2$-grams (bigrams) for this sequence would be ``the cow'', ``cow jumps",...,``the moon''. Similarly, the collection of $3$-grams (trigrams) would be ``the cow jumps'', ``cow jumps over'',....,``over the moon''.  In this case, the collection of $2$-cooccurrences and $3$-cooccurrences are same as the collection of $2$-grams and $3$-grams respectively.} The Bag-of-$L$-grams representation of a word sequence is a vector that counts the number of times any possible $\ell$-gram (from the vocabulary set $\mathcal{W}$) for some $\ell \in [L]$ appears in the sequence. In practice, $L$ is set to a small value, typically $\leq 4$. Tweaked versions of this simple representation is known to perform well for many downstream classification tasks~\citep{wang2012baselines}. It is common to ignore the ordering of words in an $\ell$-gram and to define $\ell$-gram as an unordered collection of $\ell$ words (also referred to as $\ell$-cooccurrence). Also as in~\citep{arora2018compressed} for simplicity, we assume that each $\ell$-cooccurrence contains a word at most once (as noted by~\citep{arora2018compressed} this can be ensured by merging words during a preprocessing step). 

\begin{definition} [Bag-of-$L$-cooccurrences] \label{defn:1}
Given a $k$-word sequence $S = (w_1,\dots,w_k)$ and $L \in \N$, we define its Bag-of-$L$-cooccurrences vector $\c_S$ as the concatenation of vectors $\c^{(1)},\dots,\c^{(L)}$ where 
$$\forall \ell \in [L], \c^{(\ell)} \in \R^{\binom{d}{\ell}} \mbox{ defined as }  \c^{(\ell)}  = \sum_{t=1}^{k-\ell+1} \e_{\{w_t,\dots,w_{t+\ell-1}\}},$$
Here, $\c_S$ is a $\sum_{\ell=1}^L \binom{d}{\ell}$ dimensional vector. Here, $\binom{d}{\ell}$ is number of possible $\ell$-cooccurrences in a $d$ word vocabulary set. 
\end{definition}
In practice, one would also consider only small word sequences, and therefore $k$ will be a small number. The embedding of an $\ell$-cooccurrence is defined as the elementwise product of the embeddings of its constituent words.
\begin{definition} [DisC Embedding~\citep{arora2018compressed}] \label{defn:2}
Given a $k$-word sequence  $S = (w_1,\dots,w_k)$ and $L \in \N$, we define its $L$-DisC embedding $\v_{\DisC}$ as the $pL$-dimensional vector formed by concatenation of vectors $\v^{(1)},\dots,\v^{(L)} \in \R^{p}$ where 
$$\forall \ell \in [L],  \v^{(\ell)} \in \R^p \mbox{ defined as } \v^{(\ell)} = \sum_{t=1}^{k-\ell+1} \v_{w_t} \odot \dots \odot \v_{w_{t+\ell-1}},$$
i.e., $\v^{(\ell)}$ is the sum of the $\ell$-cooccurrence embeddings of all $\ell$-cooccurrences in the document.
\end{definition}

%

From Definitions~\ref{defn:3},~\ref{defn:1}, and~\ref{defn:2}, it can be observed that if the columns of $V \in \R^{p \times d}$ are the words embeddings of the words in vocabulary set $\mathcal{W}$, then
$$ \v^{(\ell)} = V^{(\ell)} \c^{(\ell)},$$
where $V^{(\ell)}$ is the $\ell$-way column Hadamard-product constructed out of $V$. The columns of $V^{(\ell)}$ contain the DisC embedding of all possible $\ell$-cooccurrences in the vocabulary set $\mathcal{W}$ (and thus $V^{(1)} = V$).
\begin{definition} [Linearly Transformed DisC Embedding] \label{defn:4}
Given a matrix of word vectors $V \in \R^{p \times d}$ and a set of matrices $X^{(1)},\dots,X^{(L)} \in \R^{n \times p}$. A linearly transformed $L$-DisC embedding matrix $\mathbb{V}^{(L)}$ of $V$ is a block-diagonal matrix with blocks $X^{(\ell)} V^{(\ell)}, \ell \in [L]$ 
\begin{align} \label{eqn:V}
\mathbb{V}^{(L)} = \begin{bmatrix} X^{(1)}V^{(1)} & \mathbf{0}_{n \times \binom{d}{2}} & \dots  &  \mathbf{0}_{n \times \binom{d}{L}} \\
\mathbf{0}_{n \times d} & X^{(2)} V^{(2)} &  \dots  & \mathbf{0}_{n \times \binom{d}{L}} \\
\dots & \dots & \dots & \dots \\ 
\mathbf{0}_{n \times d} &  \mathbf{0}_{n \times \binom{d}{2}} & \dots &  X^{(L)}V^{(L)}
\end{bmatrix}
\end{align}
For a $k$-word sequence  $S = (w_1,\dots,w_k)$, the linearly transformed DisC embedding is then defined as $\v_{\widetilde{\DisC}} = \mathbb{V}^{(L)} \c_S$, where $\c_S$ is the Bag-of-$L$-cooccurrences vector.
\end{definition}
Note that $X^{(1)},\dots,X^{(L)}$ are fixed matrices, and can be used to encode any contextual or prior information. If they are all identity matrices, then $\v_{\widetilde{\DisC}}  = \v_{\DisC}$. 

\smallskip
\noindent\textbf{Compressed Learning Problem.} Notice that the dimensionality of $\c_S$ is $\sum_{\ell=1}^L \binom{d}{\ell}$, whereas that of $\v_{\widetilde{\DisC}}$ is only $nL$, i.e., $\v_{\widetilde{\DisC}}$ is a compressed representation of $\c_S$. Therefore, there is inherent computational advantage of working with $\v_{\widetilde{\DisC}}$ than $\c_S$, and $\v_{\widetilde{\DisC}}$ can also be computed efficiently in practice~\citep{khodak2018carte}. The question then arises is: How well does a linear classifier trained on this compressed representation perform say compared to a similar classifier trained on the uncompressed (original) representations?\footnote{As mentioned earlier, linear classifiers trained on Bag-of-$L$-grams or Bag-of-$L$-cooccurrences representations perform well in practice.}  For theoretically answering this question, we use random vectors as word embeddings, instead of pretrained vectors, as also considered by~\citet{arora2018compressed} in their analysis of DisC embeddings.\!\footnote{\citet{arora2018compressed} also provide some empirical evidence on the performance of using random signs for word embeddings.} Our analysis uses a result from~\citep{arora2018compressed} (Theorem~\ref{thm:arloss}), that bounds the loss of a linear classifier trained on the compressed domain compared to a similar classifier trained on the original domain, along with the RIP result established in Theorems~\ref{thm:rand} ($\ell=1$),~\ref{thm: RIP} ($\ell=2$), and~\ref{thm: RIP2} ($\ell \geq 3$).  To formally state the results, we need some more notation. $L$ is assumed to be a constant below. Some additional background about this problem of compressed learning is provided in Appendix~\ref{app:comp}.

We consider the standard binary classification task using labeled data. Using the notation introduced earlier, 
let $\mathcal{X}$ denote the set $\{(\c, \upsilon) \,:\, \c \mbox{ is a Bag-of-$L$-cooccurrences vector of a word sequence of length at most $k$}, \|\c\| \leq \alpha, \upsilon \in \{-1, 1\}\}$, with $\upsilon$ indicating the label on $\c$. Note that $\c$ is $h = \sum_{\ell=1}^L \binom{d}{\ell}$ dimensional vector.
Let $\mathcal{D}$ be a distribution over $\mathcal{X}$.  Consider a labeled dataset $(\c_1,\upsilon_1),\dots,(\c_b,\upsilon_b)$ drawn i.i.d.\ from $\mathcal{D}$. Consider a linear loss function, $f(\langle \c,\theta \rangle;\upsilon)$ for $\theta \in \R^h$, where $f \,: \,  \R \times \{-1,1\} \rightarrow \R$ is assumed to be convex and $\lambda$-Lipschitz in the first parameter.  Define, $f_\mathcal{D}(\theta) = \E_{(\c,\upsilon) \sim \mathcal{D}}[f(\langle \c,\theta\rangle;\upsilon)]$ as the generalization (distributional) loss on the original Bag-of-$L$-cooccurrences domain for parameter $\theta$. Let $\theta^\star \in \mbox{argmin}_{\theta \in \R^h}\, f_\mathcal{D}(\theta)$. 
Under this notation, the following corollary follows from combining Theorems~\ref{thm:rand},~\ref{thm: RIP},~\ref{thm: RIP2}, and~\ref{thm:arloss}. It shows (ignoring Lipschitz and scale parameter) that the difference in the generalization loss by operating on the compressed domain vs.\ operating on the original domain is $\approx O(\sqrt{\epsilon})$. Note that if a word sequence has length at most $k$, then all its $\ell$-cooccurrences will have cardinality at most $k$.
\begin{corollary} \label{cor:emb}
Let $V \in \R^{p \times d}$ be an i.i.d.\ centered $\tau$-bounded random matrix. Let $X^{(1)},\dots,X^{(L)} \in \R^{n \times p}$ be a set of matrices (picked independent of $V$) with $\sr(X^{(1)}) = \Omega((k \tau^4/\eps^2) \log(d/k))$, $\sr(X^{(2)}) = \Omega((k^2 \tau^8/\eps^2) \log(d^2/k))$,  $\sr(X^{(\ell)}) = \Omega((k^3 \tau^{4\ell}/\eps^2) \log(d^\ell/k))$ for $3 \le \ell \le L$. Let $\hat{\vartheta}$ be the minimizer of a classifier trained on the linearly transformed DisC embeddings $(\mathbb{V}^{(L)} \c_1,\upsilon_1),\dots,(\mathbb{V}^{(L)} \c_b,\upsilon_b)$ defined as $\hat{\vartheta} \in \mbox{argmin}_{\vartheta \in \R^{nL}}\, \frac{1}{b} \sum_i f(\langle \mathbb{V}^{(L)} \c_i,\vartheta \rangle;\upsilon_i) + \frac{1}{2C} \| \vartheta\|^2$ (for an appropriate choice of $C$), then with probability at least $(1-L\exp(-k))(1-\delta)$, we have
$$f_\mathcal{D}(\hat{\vartheta}) \leq f_\mathcal{D}(\theta^\star) + O \left (\lambda \alpha \| \theta^\star\| \sqrt{\epsilon + \frac{1}{b}\log\frac 1 \delta} \right ),$$
where $f_\mathcal{D}(\vartheta) =  \E_{(\c,\upsilon) \sim \mathcal{D}}[f(\langle \mathbb{V}^{(L)} \c,\vartheta \rangle;y)]$ is the generalization loss on the compressed domain for parameter $\vartheta$.
\end{corollary}
We end this discussion by noting that we started with random vectors for word embeddings but added the flexibility of incorporating contextual information through some fixed linear transform. While this is a good starting theoretical model for understanding the effectiveness of these embeddings, actual pretrained embeddings used in practice would not satisfy the restricted isometry property (as they will have high coherence). Bridging this gap is an interesting research direction.

\begin{small}

\end{small}
\appendix
\section{Source Separation Construction} \label{app:SS} In this section, we explain how the source separation problem can be recast into a setting where we have the product of fixed and random matrices. Let $S = MC\Phi$, where $M \in \R^{n_1 \times a}$, $C \in \R^{a \times b}$ , and $\Phi \in \R^{b \times c}$. Let $\s_1^\top,\dots \s_{n_1}^\top$ be the rows of the matrix $S$. Let $\c_1^\top,\dots,\c_a^\top$ be the rows of the matrix $C$. Let $M=(M_{ij})$. Let $A \in \R^{n_1 \times ac}$ be a matrix such that $M_{ij}$ equals $A_{i,(j-1)c+1}$ for all $i\in [n_1]$ and $j \in [a]$, and all other entries of $A$ equals $0$. It follows from this construction that:
\[
\begin{bmatrix} \s_1 \\ \vdots \\ \s_{n_1}\end{bmatrix}_{n_1c}  =  \begin{bmatrix} A  \\ \vdots \\A
\end{bmatrix}_{n_1c \times ac}  \begin{bmatrix}
\Phi^\top &  & 0 \\
& \ddots & \\
0 &  & \Phi^\top
\end{bmatrix}_{ac \times ab} \begin{bmatrix} \c_1 \\ \vdots \\ \c_a \end{bmatrix}_{ab}.
\]
Let us further define matrices $X_i \in \R^{n_1c \times c}$ as
\[ \begin{bmatrix} A  \\ \vdots \\A
\end{bmatrix} = \begin{bmatrix} X_1 | X_2| \dots | X_a \end{bmatrix}.\]
Now it is easy to observe that
\[  \begin{bmatrix} A  \\ \vdots \\A
\end{bmatrix}  \begin{bmatrix}
\Phi^\top &  & 0 \\
& \ddots & \\
0 &  & \Phi^\top
\end{bmatrix} = \underbrace{\begin{bmatrix} X_1 | X_2| \dots | X_a \end{bmatrix}}_{X}  \underbrace{\begin{bmatrix}
\Phi^\top &  & 0 \\
& \ddots & \\
 0 &  & \Phi^\top
\end{bmatrix}}_{R} = \begin{bmatrix} X_1 \Phi^\top | X_2 \Phi^\top| \dots | X_a \Phi^\top  \end{bmatrix}.  \]
Notice that if the dictionary matrix $\Phi$ is a random matrix, then each $X_i \Phi^\top$ is a product of a fixed and random matrix. Therefore, each of them fit in the product model considered in this paper. Also it follows, that if each $X_i \Phi^\top$ satisfies $(k,\epsilon)$-RIP with probability $1-\gamma$, then $\begin{bmatrix} X_1 \Phi^\top | X_2 \Phi^\top| \dots | X_a \Phi^\top  \end{bmatrix}$ satisfies $(k,\epsilon)$-RIP with probability $1-a\gamma$. 

\section{Missing Preliminaries} \label{app:prelim}
\noindent\textbf{Subgaussian and Subexponential Random Variables.} The class of subgaussian random variables is natural and quite wide. Let us start by formally defining subgaussian random variables and vectors. A number of equivalent definitions are used in the literature.
\begin{definition} [Subgaussian Random Variable and Vector]\label{def:subgauss} 
We call a random variable $x \in \R$ subgaussian if there exists  a constant $C > 0$ if $\Pr[ |x| > t]  \leq 2 \exp(-t^2/C^2)$ for all $t \ge 0$. We say that a random vector $\x \in \R^d$ is subgaussian if the one-dimensional marginals $\langle \x,\y \rangle$ are subgaussian random variables for all $\y \in \R^d$.
\end{definition}
The class of subgaussian random variables includes many random variables that arise naturally in data analysis, such as standard normal, Bernoulli, spherical, bounded (where the random variable $x$ satisfies $| x | \leq M$ {\em almost surely} for some fixed $M$). The natural generalizations of these random variables to higher dimension are all subgaussian random vectors. For many {\em isotropic convex sets}\footnote{A convex set $\mathcal{K}$  in $\R^d$ is called isotropic if a random vector chosen uniformly from $\mathcal{K}$ according to the volume is isotropic. A random vector $\x \in \R^d$ is isotropic if for all $\y \in \R^d$, $\E [ \langle \x,\y \rangle^2] = \| \y \|^2.$}
$\mathcal{K}$ (such as the hypercube), a random vector $\x$ uniformly distributed in $\mathcal{K}$ is subgaussian.

\begin{definition}[$\psi_2$-norm of a Subgaussian Random Variable and Vector]
The $\psi_2$-norm of a subgaussian random variable $x \in \R$, denoted by $\| x \|_{\psi_2}$ is:
$$ \| x \|_{\psi_2} =  \sup_{a \geq 1}a^{-1/2} (\E[|x|^a])^{1/a}. $$
The $\psi_2$-norm of a subgaussian random vector $\x \in \R^d$ is:
$$ \| \x \|_{\psi_2} = \sup_{\y \in S^{d-1}} \; \| \langle \x,\y \rangle \|_{\psi_2}.$$
\end{definition}

We also will work with a class of subexponential random variables, those with at least an exponential tail.
\begin{definition} [Subexponential Random Variable  and $\psi_1$-norm]
A random variable $x$ that satisfies $\Pr [|x| > t] \leq 2 \exp(-t/C)$ for all $t \ge 0$ is called a subexponential random variable. The subexponential norm of $x$, denoted $\| x \|_{\psi_1}$, is defined as $\| x \|_{\psi_1} = \sup_{a \geq 1}a^{-1} (\E[|x|^a])^{1/a}$.
\end{definition}

An immediate consequence from the above definitions of subgaussian and subexponential random variables is that,
\begin{eqnarray*}
&\mbox{For subgaussian } x: \; (E[|x|^a])^{1/a} \leq \|x\|_{\Psi_2} \sqrt{a},  \; \forall a \geq 1& \\
&\mbox{For subexponential } x: \; (E[|x|^a])^{1/a} \leq \|x\|_{\Psi_1} a, \; \forall a \geq 1. & 
\end{eqnarray*}

\noindent\textbf{Corollaries from Hanson-Wright Inequality (Theorem~\ref{thm: HW}).} A simple corollary of Hanson-Wright inequality from Theorem~\ref{thm: HW} is a concentration inequality for random vectors with independent subgaussian components.
\begin{corollary} [Subgaussian Concentration~\citep{RVHanson-Wright}]  \label{cor:subgauss}
Let $M$ be a fixed $n \times d$ matrix. Let $\x = (x_1,\ldots,x_n) \in \R^n$ be a random vector with independent components $x_i$ which satisfies $\E[x_i] = 0$, $\E[x_i^2] = 1$  and $\|x_i\|_{\psi_2} \leq K$.  Then  for every $t \ge 0$,
$$ \Pr[|\| M \x \| - \| M \|_F| > t] \leq 2\exp\left ( \frac{-ct^2}{K^4 \| M \|^2} \right ).$$
\end{corollary}
Another corollary of this inequality is a bound on the spectral norm of product of deterministic and random matrices.
\begin{corollary}[Spectral Norm of the Product~\citep{RVHanson-Wright}]	\label{cor: product norm}
Let $B$ be a fixed $n \times p$ matrix, and let $G=(G_{ij})$ be a $p \times d$ random matrix with independent entries that satisfy: $\E[G_{ij}] = 0$, $\E[G_{ij}^2] = 1$, and $\|G_{ij}\|_{\psi_2} \leq K$. Then for any $a,b > 0$,
\[ \Prob{ \norm{BG} > C K^2 (a\norm{B}_\HS+  b\sqrt{d} \norm{B}) } \le 2 \exp(- a^2\sr(B) -  b^2 d) \]
\end{corollary}

\noindent\textbf{RIP and Compressed Sensing.} Common random matrices such as Gaussian, Bernoulli, Fourier all satisfy the restricted isometry property. For example, if the entries of $M \in \R^{n \times d}$ are independent and identically distributed subgaussian random variables with zero mean and unit variance. Assume that $n \geq Ck \log(2d/k)$ where C depends only on $\epsilon,\gamma$, and the subgaussian moment. Then with probability at least $1 - \gamma$, the matrix $\tilde{M} = M/\sqrt{n}$ satisfies the restricted isometry property.

The following classic result, shows that one can recover from compressed signal (even in presence of noise) if the measurement matrix satisfies RIP.
\begin{theorem}[[Theorem 1.2,~\citep{candes2008restricted}] \label{thm:candes}
Suppose that $M$ satisfies the $(2k,\epsilon)$-RIP with $\epsilon < \sqrt{2}-1$ and let $\y = M\x + \e$ where $\| \e \| \leq \delta$. Let $\hat{\x} = \mbox{argmin}_{\z}\, \|\z\|_1$ subject to $\| M\z - \y \| \leq \delta$. Then $\hat{\x}$ obeys
$$\| \hat{\x} - \x \| \leq C_0 \frac{\sigma(\x)_1}{\sqrt{k}} + C_2 \delta \;\;\;\; \mbox{ where},$$
$$C_0 = 2\frac{1-(1-\sqrt{2})\epsilon}{1-(1+\sqrt{2})\epsilon}, \;\;\;\; C_2 = 4\frac{\sqrt{1+\epsilon}}{1-(1+\sqrt{2})\epsilon},  \;\;\;\;\mbox{and } \sigma(\x)_1 = \min_{k\mbox{-sparse} \;\; \x'}\; \|\x -\x'\|_1.$$
\end{theorem}

\noindent\textbf{Restricted Isometry under the $XR$- vs. the $RX$-model.} Our results establish the restricted isometry condition for a matrix that can be factorized as $XR$, for a deterministic $X$ and a random $R$, under some conditions on $X$ and $R$. One could ask what happens if we change the model and work in a setting where the factorization is of the form $RX$ (again for random $R$, and deterministic $X$). Unfortunately, in this case one could not hope to get a mild assumption on $X$. In fact, a necessary and sufficient condition on $X$ for a matrix $RX$ to satisfy RIP is that $X$ satisfies RIP. The necessary condition is trivial. Indeed, if $\|X\u\|$ is essentially different from $\|X\v\|$ for two sparse unit vectors $\u,\v$, then $\E[\|R X \u\|]$ will be essentially different from $\E[\|R X \v\|]$ (unless, of course, $X$ satisfies RIP). The proof for the sufficiency is not much harder, and follows from standard techniques in the literature.

\section{Missing Details from Section~\ref{sec:randR}} \label{app:randR}
We provide missing details from Section~\ref{sec:randR}. We start with the proof of Theorem~\ref{thm:rand}.

\begin{theorem}[Theorem~\ref{thm:rand} Restated]
Let $X$ be an $n \times p$ matrix. Let $R=(R_{ij})$ be a $p \times d$ matrix whose entries are with independent entries such that $\E[R_{ij}] = 0$, $\E[R_{ij}^2] = 1$, and $\|R_{ij}\|_{\psi_2} \leq K$. Let $\epsilon \in (0,1)$, and let $k \in \N$ be a number satisfying $\sr(X)\ge CK^4 \frac{k}{ \epsilon^2} \log \left ( \frac{d}{k} \right )$. Then with probability at least $1-\exp(-c \epsilon^2 \sr(X)/K^4)$, the matrix $XR$ satisfies $(k, \epsilon)$-RIP, i.e.,
 \[\forall \u \in \Sigma_k, \; (1-\epsilon) \norm{X}_{\HS} \le \norm{XR\u} \le (1+\epsilon) \norm{X}_{\HS}.\]
\end{theorem}
\begin{proof}
First, we prove a spectral norm bound for every $n \times k$ submatrix of $XR$. Fix a subset $I \subset [d]$ with $|I| = k$. Let $R_I$ be the matrix $R$ restricted to the columns indexed by $I$. Using Corollary~\ref{cor: product norm} (with $a=\eps/K^2,b=\eps/K^2$), $\| X R_I\|$ satisfies
\[ \Prob{ \| X R_I\| > C_1 \eps (\norm{X}_\HS+ \sqrt{k} \norm{X}) } \le 2 \exp(-\sr(X)\eps^2/K^4 -  k\eps^2/K^4). \]
By assumption, $\sr(X) \geq  k \log(d/k)$, therefore $\| X \|_{F} \geq \sqrt{k \log (d/k)} \| X \| \geq \sqrt{ k} \| X \|$. Then
\[ \Prob{ \| X R_I\| > 2C_1\eps \norm{X}_\HS} \le 2 \exp \left (-\sr(X)\eps^2/K^4  \right ). \]
Using a union bound, and the fact that  $\binom{d}{k} \leq \left (\frac{d e}{k} \right )^k$,
\begin{align} \label{eqn:XRrandspec}
& \Prob{ \exists I \in [d], \ |I| = k, \ \| X R_I\| > 2C_1\eps \norm{X}_\HS}  \le  \binom{d}{k} \cdot 2 \exp \left ( \sr(X)\eps^2/K^4 \right ) \nonumber \\
& \le 2 \exp \left (-\sr(X)\eps^2/K^4 + k \log(e d/k) \right ) \le 2 \exp(-c_0 \sr(X)\eps^2/K^4),
\end{align}
for appropriate choice of constants.

Let us now fix an $\u \in S^{d-1}$. A simple consequence of the Hanson-Wright inequality (see Theorem 3.2, \citep{RVHanson-Wright}) is that for any $t \geq 0$,
\[ \Pr \left [ \left | \| XR \u \| - \| X \|_F \right | \geq t \right ] \leq 2 \exp\left ( \frac{-c_1 t^2}{K^4 \| X\|^2} \right ).\]
Setting $t =  \epsilon \| X \|_F$, we get
\[ \Pr \left [ \left | \| XR \u \| - \| X \|_F \right | \geq \epsilon \| X \|_F \right ] \leq 2 \exp \left (\frac{-c_1 \epsilon^2 \sr(X)}{K^4} \right ).\]
Let $\Sigma_k$ be the set of all $k$-sparse vectors in $S^{d-1}$. 
By a volumetric estimate, this set has an $(1/2C_2)$-net $\NN$ of cardinality smaller than 
\[\binom{d}{k} \cdot  (6C_2)^k \le \exp \left( k \log \frac{C_0d}{k} \right).\]
Taking the union bound over this net, we obtain
\begin{align} \label{eqn:event1}
&\Pr \left[ \forall \u \in \NN, \ \ | \norm{XR\u} - \norm{X}_{\HS} | \le \epsilon \norm{X}_{\HS} \right ] \ge 1- \exp \left( \frac{-c_1 \epsilon^2 \sr(X)}{K^4}+ k \log \frac{C_0d}{k} \right).
\end{align}
We now use a simple approximation idea for extending the above argument from the net to all $k$-sparse vectors. Let us first assume that events described in~\eqref{eqn:event1} and~\eqref{eqn:XRrandspec} happen. We can write any $\u \in \Sigma_k$ as $\u = \a + \b$, where $\a$ in $\NN$, and $\b$ is such that $|\supp(\b)| \leq k$ and $\| \b \| \leq 1/(2C_2)$. Let $I_\b = \supp(\b) \subset [d]$.  Let $\bm{\tilde{b}}$ be $\b$ restricted to $\supp(\b)$.
\begin{align*}
\|X R \u \| & = \| X R \a + X R \b \| \leq \| X R \a \| + \| X R \b \| = \| X R \a \| + \| X R_{I_\b} \bm{\tilde{b}} \| \\
& \leq \| X R \a \| + \| X R_{I_\b}\| \| \bm{\tilde{b}} \| \leq (1+\epsilon) \| X \|_F + \frac{1}{2C_2}  \| X R_{I_\b} \| \\
& \leq (1+\epsilon_1) \| X \|_F,
\end{align*}
where for the spectral norm bound for $\| X R_{I_\b} \|$ we use~\eqref{eqn:XRrandspec} and the bound on $\| X R \a \|$ follows from~\eqref{eqn:event1}. Similarly,
\begin{align*}
\|X R \u \| \geq (1 - \epsilon) \| X \|_F - \frac{\epsilon}{2C_2}  \| X R_{I_\b} \| \geq (1-\epsilon_2) \| X \|_F. \end{align*}
Adjusting the constants, and removing the conditioning completes the proof of the theorem.
\end{proof}

\subsection{Restricted Isometry of $XR$ with Low Randomness $R$} \label{sec:lowrandom}
In this section, we operate under a weaker randomness assumption on $R$. In particular, we will use the notion of $l$-wise independence to capture low randomness to construct $R$. When truly random bits are costly to generate or supplying them in advance requires too much space, the standard idea is to use $l$-wise independence which allows one to maintain a succinct data structure for storing the random bits. 
\begin{definition} \label{def:lwise}
A sequence of random variables $x_1,\dots,x_m$ is called $l$-wise independent if every $l$ of them are independent. More formally, $x_1,\dots,x_m$ drawn from some distribution $\mathcal{D}$ over a range $\Upsilon$ if for all $i_1,i_2,\dots,i_l$ (all unique) and $t_1,\dots ,t_l \in \Upsilon$,
\[ \Pr_{x_1,\dots,x_m \sim \mathcal{D}}[x_{i_1} = t_1, \dots, x_{i_l} = t_l] = \Pr[x_{i_1} = t_1] \cdots \Pr[x_{i_l} = t_l]. \]
\end{definition}
For simplicity, we will work with Rademacher random variables (random signs). Constructing $n$ $l$-wise independent random signs from $O(l \log(n))$ truly independent random signs using simple families of hash functions is a well-known idea~\citep{motwani1995randomized}. 

Let $R$ satisfy $2\sr(X)$-wise independence. Our general strategy in this proof will be to rely on higher moments where we can treat certain variables as independent.  We start with $Q \in \R^{p \times k}$ with i.i.d.\ $\pm 1$ entries, and establish an bound on $\E[\| X Q \|^{2\sr(X)}]$. Using Markov's inequality for higher moments and a union bound gives our first result, $\Pr \left[\exists I \subset [d], \ |I|=k, \ \norm{XR_I} > C_0 \eps \norm{X}_{\HS} \right]  \le \exp(-c_0  \eps^2 \sr(X))$ (Lemma~\ref{lem:wisespec}). We then derive a concentration bound for $\| X R \u\|$ for a fixed $\u \in S^{d-1}$. For this, we investigate certain higher moments of $\|XR \u \|^2 - \E[\|XR \u \|^2 ]$ using moment generating functions. The result then follows using a net argument over the set of sparse vectors on the sphere.

\begin{lemma} \label{lem:wisespec}
 Let $X$ be an $n \times p$ matrix and $\eps \in (0,1)$. Let $R$ be a $p \times d$ matrix whose entries are $2 \sr(X)$-wise independent $\pm 1$-random variables. Let $k \in \N$ be a number satisfying $\sr(X)\ge C \frac{k}{\eps^2} \log \left ( \frac{d}{k} \right )$. Then 
\[\Pr \left[ \forall I \subset [d], \ |I|=k, \ \norm{XR_I} \ge C_0 \eps \norm{X}_{\HS} \right] \le \exp(-c_0 \eps^2 \sr(X)), \]
where $R_I$ the $p \times k$ submatrix of $R$ with columns from the set $I$. 
\end{lemma}
\begin{proof}
Denote for shortness $m= \sr(X)$. First let us fix $I \subset [d],  |I|=k $ and introduce a $p \times k$ matrix $Q$ with i.i.d.\ $\pm 1$ entries. Then  expanding the trace of the power of a matrix and using the $2m$-wise independence,
\[ \E[\norm{XR_I}^{2m}] \leq  \E[\norm{XR_I}_F^{2m}]   \le \E[\tr (R_I^\top X^\top X R_I)^m]  = \E [\tr (Q^\top X^\top X Q)^m]  = \E [\| XQ \|_F^{2m}]  \le k \E [\norm{XQ}^{2m}].\]
Here, we use the fact that since expectation and trace are both linear, they commute, and also the fact the Frobenius norm is at most the rank times the spectral norm.

Since the matrix $Q$ is subgaussian, the last expectation can be estimated using Corollary~\ref{cor: product norm} (with $a=b=s$) as
\[ \Pr \left[ \norm{XQ} \ge C_1s(\norm{X}_{\HS}+ \sqrt{k} \norm{X}) \right] \le 2 \exp (-s^2 \sr(X) -s^2 k).\]
By the stable rank assumption on $X$, $\norm{X}_{\HS} \geq \sqrt{k} \norm{X}$. Setting $s = a/(2C_1\|X\|_F)$,
\[ \Pr \left [ \norm{XQ} \ge a \right ] \le 2 \exp(-c a^2/\| X \|^2).\]
Using this we get,
\begin{align} \label{eqn:uni}
\E[\norm{XQ}^{2m}]
&\le (C_1 \norm{X}_{\HS} )^{2m} +2\int_{C_1 \norm{X}_{\HS}}^{\infty} m a^{2m-1} \Pr \left[ \norm{XQ} \ge a \right] \, da  \nonumber \\
&\le (C_1 \norm{X}_{\HS} )^{2m} +2 \int_0^{\infty} m a^{2m-1} \exp \left( -c \frac{a^2}{\norm{X}^2} \right) \, da \nonumber \\
&\le (C_1 \norm{X}_{\HS} )^{2m} + C_2^{2m} m^m \norm{X}^{2m} \nonumber \\
& =  (C_1 \norm{X}_{\HS} )^{2m} +C_2^{2m} \left ( \frac{\|X\|_F}{\|X\|} \right )^{2m} \norm{X}^{2m} \nonumber\\
& \le (C_1' \norm{X}_{\HS} )^{2m},
\end{align}
where the last inequality follows from the assumption on $\sr(X)$. Using Markov's inequality and the union bound, we obtain
\begin{align*}  
\Pr \left[\exists I \subset [d], \ |I|=k, \ \norm{XR_I} > \eps C_1' \norm{X}_{\HS} \right]  &\le \binom{d}{k} \cdot \max_{|I|=k} \frac{\E[\norm{XR_I}^{2m}]}{(\eps C_1' \norm{X}_{\HS})^{2m}}  \\
& \le  k \exp \left( k \log \left( \frac{ed}{k} \right) - 2\eps^2 m \right)
\le \exp(-c_0 \eps^2 m),
\end{align*}
where we used the uniform bound on $\E[\norm{XR_I}^{2m}]$ as established in~\eqref{eqn:uni}.
\end{proof}

\begin{theorem}\label{thm:lownoise}
Let $X$ be an $n \times p$ matrix. Let $R$ be a $p \times d$ matrix whose entries are $2 \sr(X)$-wise independent $\pm 1$-random variables. Let $\epsilon \in (0,1)$, and let $k \in \N$ be a number satisfying $\sr(X)\ge C \frac{k}{\epsilon^2} \log \left ( \frac{d}{k} \right )$.
 Then with probability at least $1-\exp(-c \epsilon^2 \sr(X))$, the matrix $XR$ satisfies $(k, \epsilon)$-RIP, i.e.,
 \[ \forall \u \in \Sigma_k, \; (1-\epsilon) \norm{X}_{\HS} \le \norm{XR\u} \le (1+\epsilon) \norm{X}_{\HS}. \]
\end{theorem}
\begin{proof}
We derive a small ball probability estimate.  Fix $I \subset [d], \ |I|=k$, and let $\u \in S^{d-1}$ be a vector with $\supp(\u) \subset I$. Let $l = c \epsilon^2 \sr(X)$ and let $Q$ be a $p \times k$ matrix $Q$ with i.i.d.\ $\pm 1$ entries (as in Lemma~\ref{lem:wisespec}). Then
\[ \E[\u^{\top} R^{\top} X^{\top} X R \u] = \E[\u^{\top} Q^{\top} X^{\top} X Q \u] =\norm{X}_{\HS}^2. \]
Also the $2 \sr(X)$-wise independence implies that 
\begin{align} \label{eqn:Ru}
 \E\left[ (\u^{\top} R^{\top} X^{\top} X R \u - \E[\u^{\top} R^{\top} X^{\top} X R \u])^{2l} \right. ] = \E \left[ \left ( \u^{\top} Q^{\top} X^{\top} X Q \u - \E[\u^{\top} Q^{\top} X^{\top} X Q \u] \right )^{2l} \right ].
\end{align}
Here $Q\u$ is a vector with independent centered subgaussian coordinates. Let $\z = Q \u$, and $A = X^\top X$. Let $A=(A_{ij})$. We can rewrite $\u^{\top} Q^{\top} X^{\top} X Q \u - \E[\u^{\top} Q^{\top} X^{\top} X Q \u]$ as $\z^\top A \z - \E[\z^\top A \z]$. By centering and independence of the entries in $\z$, we can represent, with $\z=(z_1,\dots,z_p)$.
\[ \z^\top A \z - \E[\z^\top A \z] = \sum_{ij} A_{ij} z_i z_j - \sum_{i} A_{ii} \E[z_i^2] = \sum_{i} A_{ii} (z_i^2 - \E[z_i^2]) + \sum_{i,j, i\neq j} A_{ij} z_i z_j.\]
We are interested in bounding $\E[(\z^\top A \z - \E[\z^\top A \z])^{2l}]$. Here,
\begin{align} \label{eqn:11}
\E \left [ \left (\sum_{i} A_{ii} (z_i^2 - \E[z_i^2]) + \sum_{i,j, i\neq j} A_{ij} z_i z_j \right )^{2l} \right ]  \leq 2^{2l} \left( \E \left [  \left ( \sum_{i} A_{ii} (z_i^2 - \E[z_i^2]) \right )^{2l}  \right ]  +\E \left [ \left ( \sum_{i,j, i\neq j} A_{ij} z_i z_j \right )^{2l} \right ] \right ) .
\end{align} 

\begin{CompactEnumerate}
\item Let us first focus on bounding, $\E [ \sum_{i} (A_{ii} (z_i^2 - \E[z_i^2])  )^{2l}  ]$. Since, $z_i^2 - \E[z_i^2]$ are all independent mean-zero subexponential random variables\footnote{A random variable $x$ is subgaussian if and only if $x^2$ is subexponential, therefore $\| x \|^2_{\Psi_2} \leq \| x^2 \|_{\Psi_1} \leq 2 \| x\|^2_{\Psi_2}$.}, and
\[ \| z_i^2 - \E[z_i^2] \|_{\Psi_1} \leq 2 \| z_i^2 \|_{\Psi_1} \leq 4 \| z_i \|^2_{\Psi_2} = 4,\]
as for $\pm 1$-random variables the $\Psi_2$-norm is $1$.
Let $S_0 =  \sum_{i} A_{ii} (z_i^2 - \E[z_i^2])$. The moment generating function of $S_0$ can be bound using standard techniques (see~\citep[Proposition 5.16]{V11}). We get that $| \lambda | \leq c_1/\max_i |A_{ii}|$,
\[\E[\exp(\lambda S_0)] \leq \exp(c_2 \lambda^2 \sum_{i=1}^p A_{ii}^2) .\]
Now using subexponential tail estimates yields,
\begin{align} \label{eqn:12}
\E \left [ \left ( \sum_{i} A_{ii} (z_i^2 - \E[z_i^2])  \right)^{2l}  \right ] \leq (C_3' \|A\|  \cdot 2l)^{2l} = (C_3 \|X\|^2 l)^{2l} .
\end{align}


\item We now bound $\E \left [ ( \sum_{i,j, i\neq j} A_{ij} z_i z_j  )^{2l} \right ]$. Let $S = \sum_{i,j, i\neq j} A_{ij} z_i z_j$. For bounding  $\E[S^{2l}]$, we use the following result implicit in~\citep{RVHanson-Wright}. The proof is based on decoupling and reduction to normal random variables arguments and is omitted here. 
\begin{claim} [From Theorem 1.1,~\citep{RVHanson-Wright}] \label{claim:rv}
Let $A=(A_{ij})$ be a $p \times p$ matrix. Let $x_1,\dots,x_p$ be a sequence of independent random variables such that $\E[x_i] = 0$, $\E[x_i^2] = 1$, and $\|x_i\|_{\psi_2}$ is bounded. Let $S = \sum_{i,j, i\neq j} A_{ij} x_i x_j$. Then the moment generating function $\E[\exp(\lambda S)] \leq \exp(C \lambda^2 \| A \|_F^2)$ for all $\lambda \leq c/\|A\|$.
\end{claim}
Claim~\ref{claim:rv} along with subgaussian tail estimates yields,
\begin{align} \label{eqn:13}
\E \left [  \left ( \sum_{i,j, i\neq j} A_{ij} z_i z_j  \right)^{2l} \right ] = \E[S^{2l}] = (C_4' \|A\|_F \cdot \sqrt{2l})^{2l} \leq (C_4  \|X\| \|X\|_F \sqrt{l})^{2l}.
\end{align}
Here for the last inequality we used that $\|A\|_F = \|X^\top X \|_F \leq \|X\| \|X\|_F$.
\end{CompactEnumerate}
Plugging in the bounds from~\eqref{eqn:12} and~\eqref{eqn:13} in~\eqref{eqn:11} yields,
\[ \E \left [ \left (\sum_{i} A_{ii} (z_i^2 - \E[z_i^2]) + \sum_{i,j, i\neq j} A_{ij} z_i z_j \right )^{2l} \right ] \leq (C_5 \sqrt{l} \|X\| \|X\|_F)^{2l},\]
as $l < \sr(X)$. In other words (from~\eqref{eqn:Ru}),
\[ \E\left[ (\u^{\top} R^{\top} X^{\top} X R \u - \E[\u^{\top} R^{\top} X^{\top} X R \u])^{2l} \right ]  = \E \left[ \left(\u^{\top} R^{\top} X^{\top} X R \u - \norm{X}_{\HS}^2 \right )^{2l} \right ] \leq (C_5 \sqrt{l} \|X\| \|X\|_F)^{2l}.\]
Combining this with Markov's inequality,
\begin{align*}
 \Pr \left[ \left | \norm{XR\u} - \norm{X}_{\HS} \right | \ge \epsilon' \norm{X}_{\HS} \right] & \le  \Pr \left[ \left | \norm{XR\u}^2 - \norm{X}_{\HS}^2 \right | \ge \epsilon'^2 \norm{X}_{\HS}^2 \right] \\
 & = \Pr \left[ \left |  \u^{\top} R^{\top} X^{\top} X R \u - \norm{X}_{\HS}^2 \right | \ge  \epsilon \norm{X}_{\HS}^2 \right] \\
 &\le \frac{(C_5 \sqrt{l} \norm{X}_{\HS} \norm{X})^{2l}}{(\epsilon \norm{X}_{\HS}^2)^{2l}}
 \le \exp(-2l),
\end{align*}
if the constant $c$ in the definition of $l$ is chosen sufficiently small, and setting $\epsilon'^2 = \epsilon$. For the first inequality, we used the fact that for $a_1,a_2,a_3 \geq 0$, $\Pr[|a_1^2 - a_2^2| \geq a_3^2] \geq \Pr[ |a_1 - a_2|^2 \geq a_3^2] = \Pr[|a_1-a_2| \geq a_3]$.

We now finish the proof using a net argument as in Theorem~\ref{thm:rand}. Let $\NN$ be an $(1/2C_2)$-net $\NN$ over $\Sigma_k$. As mentioned in Theorem~\ref{thm:rand},
\[ | \NN | \leq  \exp \left( k \log \frac{C_0 d}{k} \right ) .\]
Taking the union bound over this net, we obtain
\begin{align*}
\Pr \left[ \forall \u \in \NN, \ \ | \norm{XR\u} - \norm{X}_{\HS} | \le \epsilon' \norm{X}_{\HS} \right]  \ge 1- \exp \left(-2l + k \log \frac{C_0d}{k} \right) \ge 1- \exp(-l),
\end{align*}
where we used our choice of $l$ in the last inequality.  

Reinitializing $\eps$ and using the spectral norm bound from Lemma~\ref{lem:wisespec} in conjunction with the approximation idea used in Theorem~\ref{thm:rand} yields that with probability at least $1-\exp(-c\epsilon^2 \sr(X))$, $| \norm{XR\u} - \norm{X}_{\HS} | \le \epsilon \norm{X}_{\HS}$ for all $\u \in \Sigma_k$. The theorem is proved.
\end{proof}

\noindent\textbf{Comparing Theorems~\ref{thm:lownoise} and~\ref{thm:rand}.}  While the assumption on the stable rank does not change (by more than a constant) between these two theorems, we have a drastic reduction in the number of random bits from $O(pd)$ (in Theorem~\ref{thm:rand}) to $O(\sr(X) \log (pd))$ (in Theorem~\ref{thm:lownoise}). Storing the $XR$ matrix when $R$ is an i.i.d.\ random matrix takes $O(np+pd)$ words of memory, while if $R$ is $2\sr(X)$-wise independent storing $XR$ only requires $O(np+\sr(X) \log (pd))$ words of memory.

%
%

\subsection{Restricted Isometry of $XR$ with Sparse Random $R$} \label{sec:sparserandom}
In this section, we investigate the restricted isometry property when $R$ is a sparse random matrix. We use the following popular probabilistic model for our sparse random matrices~\citep{spielman2012exact,luh2016dictionary,wang2016blind}.

\begin{definition} [Bernoulli-Subgaussian Sparse Random Matrix] \label{def:sparse}
We say that $R \in \R^{p \times d}$ satisfies the Bernoulli-Subgaussian model with parameter $\beta \in (0,1)$ if $R = \Omega \odot \Gamma$, where $\Omega \in \R^{p \times d}$ is an i.i.d.\ Bernoulli matrix where each entry is $1$ independently with probability $\beta$, and $\Gamma=(\Gamma_{ij})$ is an random matrix with independent entries that satisfy: $\E[\Gamma_{ij}] = 0$, $\E[\Gamma_{ij}^2] = 1$, and $\|\Gamma_{ij}\|_{\psi_2} \leq K$, and $\odot$ denotes the Hadamard product.
\end{definition}
We note that the sparsity of $XR$ is manipulated by the Bernoulli distribution, and the non-zero entries of $XR$ obey the subgaussian distribution, thereby facilitating a very general model of the sparse matrix.

Our results in this section will rely on the sparse Hanson-Wright inequality from~\citet{zhou2015sparse} (Theorem~\ref{thm:sparseHS}). Similar to Hanson-Wright inequality (Theorem~\ref{thm: HW}), sparse Hanson-Wright inequality provides a large deviation bound for a quadratic form. However, in this case, the quadratic form is sparse, and is of the form $(\x \odot \zeta)^\top M (\x \odot \zeta)$ where $\x$ is an random vector with
independent centered subgaussian components and $\zeta$ contains independent Bernoulli random variables.


\begin{theorem} [Restated from Theorem 1.1~\citep{zhou2015sparse}] \label{thm:sparseHS}
Let $\x = (x_1,\dots,x_d) \in \R^d$ be a random vector with independent components $x_i$ which satisfy $\E[x_i] = 0$ and $\| x_i \|_{\Psi_2} \leq K$ is bounded. Let $\zeta = (\zeta_1,\dots,\zeta_d) \in \{0,1\}^d$ be a random vector independent of $\x$, with independent Bernoulli random variables $\zeta_i$ such that $\E[\zeta_i] = \beta$. Let $M = (M_{ij})$ be an $d \times d$ matrix. Then  for every $t \ge 0$,\footnote{The expectation $\E[]$ is now over both the randomness in $\x$ and $\zeta$.}
\begin{multline*}
\Pr \left [ \left | (\x \odot \zeta)^\top M (\x \odot \zeta) - \E[(\x \odot \zeta)^\top M (\x \odot \zeta)] \right | > t \right ]  \leq \\ 2\exp \left (-c \min \left \{ \frac{t^2}{K^4 (\beta \| \diag(M) \|_F^2+\beta^2 \| \offdiag(M) \|_F^2)}, \frac{t}{K^2 \|M\|} \right \} \right ) \leq 2\exp \left (-c \min \left\{ \frac{t^2}{\beta K^4 \| M \|_F^2}, \frac{t}{K^2 \|M\|} \right \} \right )
\end{multline*}
\end{theorem}
We now show that using the above sparse Hanson-Wright inequality (Theorem~\ref{thm:sparseHS}), one can obtain a concentration inequality for sparse random vectors. The following simple lemma follows a proof strategy as that used in Corollary~\ref{cor:subgauss}.
\begin{lemma} \label{lem:sparseconc}
Let $M$ be a fixed $n \times d$ matrix. Let $\x$ and $\zeta$ be random vectors as in Theorem~\ref{thm:sparseHS}. Then for any $\epsilon \geq 0$,
\[  \Pr \left [ \left | \| M (\x \odot \zeta) \| - \sqrt{\beta} \| M \|_F \right | > \epsilon \sqrt{\beta} \| M \|_F \right ] \leq 2 \exp \left (\frac{-c \beta \epsilon^2 \sr(X)}{K^4} \right).\]
\end{lemma}
\begin{proof}
Let $\y = \x \odot \zeta$. Let $Q = M^\top M$. We apply Theorem~\ref{thm:sparseHS} for random vector $\y$ and matrix $Q$. We have, $\y^\top Q \y = \| M \y \|^2$ and $\E[\y^\top Q \y] = \beta \| M \|_F^2$. Thus, we obtain for any $t \ge 0$, 
\[\Pr \left [ \left | \| M \y \|^2 - \beta \| M \|_F^2 \right | > t \right ] \leq 2\exp \left (-c \min \left\{ \frac{t^2}{\beta K^4 \| M^\top M \|_F^2}, \frac{t}{K^2 \|M\|^2} \right \} \right ).\]
We now can use the fact that, $\|M^\top M \|_F \leq  \|M^\top \| \|M\|_F = \|M \| \|M\|_F$. Setting $t = \rho \beta \| M \|_F^2$. It follows that
\[ \Pr \left [ \left | \| M \y \|^2 - \beta \| M \|_F^2 \right | > \rho \beta \| M \|_F^2 \right ] \leq 2\exp \left (-c \beta \min\{\rho^2,\rho\} \frac{\|M\|_F^2}{K^4 \| M \|^2}  \right ).\]
Consider the event $| \| M \y \|^2 - \beta \| M \|_F^2 | \leq \rho \beta \| M \|_F^2$. Dividing both sides by $\beta \| M \|_F^2$ gives, 
\begin{align} \label{eqn:event} \left | \frac{\| M \y \|^2}{\beta \| M \|_F^2} - 1 \right | \leq \rho.\end{align}
Now let $\rho=\max\{\epsilon,\epsilon^2\}$. Note that for any $r \geq 0$, $\max\{ |r-1|, |r-1|^2\} \leq |r^2 -1|$. Let $r^2 = \frac{\| M \y \|^2}{\beta \| M \|_F^2}$.
If $\rho = \epsilon$, then~\eqref{eqn:event} implies 
\[\left | \frac{\| M \y \|}{\sqrt{\beta} \| M \|_F} - 1 \right | \leq \left | \frac{\| M \y \|^2}{\beta \| M \|_F^2} - 1 \right | \leq \epsilon. \]
If $\rho = \epsilon^2$, then~\eqref{eqn:event} implies 
\[\left | \frac{\| M \y \|}{\sqrt{\beta} \| M \|_F} - 1 \right |^2 \leq \left | \frac{\| M \y \|^2}{\beta \| M \|_F^2} - 1 \right | \leq \epsilon^2 \Longrightarrow \left | \frac{\| M \y \|}{\sqrt{\beta} \| M \|_F} - 1 \right | \leq \eps. \]
Putting together the above two inequalities shows that under the event $| \| M \y \|^2 - \beta \| M \|_F^2 | \leq \rho \beta \| M \|_F^2$, implies $ | \| M \y \| - \sqrt{\beta} \| M \|_F | \leq \epsilon \sqrt{\beta} \| M \|_F$. Using this along with the observation that $\min\{\rho^2,\rho\} = \epsilon^2$, we get the claimed result,
\[  \Pr  \left [ \left | \| M \y \| - \sqrt{\beta} \| M \|_F \right | > \epsilon \sqrt{\beta} \| M \|_F \right ] \leq 2 \exp \left (-c \beta \epsilon^2 \frac{\|M\|_F^2}{K^4 \| M \|^2} \right).\]
\end{proof}

\begin{lemma}\label{lem:sparsefixed}
Let $X$ be a fixed $n \times p$ matrix, and let $R =  \Omega \odot \Gamma $ be a sparse random matrix as in Definition~\ref{def:sparse}. Let $\epsilon > 0$. Then for any fixed $\u \in S^{d-1}$,
\[ \Pr \left [ \left | \| XR \u \| - \sqrt{\beta} \| X \|_F \right | > \epsilon \sqrt{\beta} \| X \|_F\right ]  \leq 2 \exp \left (\frac{-c \beta \epsilon^2 \sr(X)}{K^4} \right).\]
\end{lemma}
\begin{proof}
Let us first fix $\u \in S^{d-1}$. By concatenating the rows of $R$, we can view $R$ as a long vector in $\R^{pd}$. Consider the linear operator $T \,:\, L_2^{pd} \rightarrow L_2^n$ defined as $T(R) = XR\u$. Note that here $\u$ is a fixed unit vector. Let us apply Lemma~\ref{lem:sparseconc} to the linear operator $T(R)$. The Frobenius norm of $T$ equals $\|X\|$, and the operator norm of $T$ can be bound as $\| X R \u\| \leq \| X \| \| R\|_F \| \u \| \leq \| X \|$   (as $\|R\|_F$ is the Euclidean norm of $R$ as a vector in $L_2^{pd}$). From Lemma~\ref{lem:sparseconc}, for any $\epsilon > 0$, we get the claimed bound,
\[ \Pr \left [ \left | \| XR \u \| - \sqrt{\beta} \| X \|_F \right | > \epsilon \sqrt{\beta} \| X \|_F\right ]  \leq 2 \exp \left (\frac{-c \beta \epsilon^2 \sr(X)}{K^4} \right).\]
\end{proof}

\begin{lemma}\label{lem:sparsenrom}
Let $X$ be a fixed $n \times p$ matrix, and let $R =  \Omega \odot \Gamma $ be a sparse random matrix as in Definition~\ref{def:sparse}. Let $\epsilon > 0$. Let $k \in \N$ be a number satisfying $\sr(X)\ge C \frac{K^4}{\eps^2} \frac{k}{\beta} \log \left ( \frac{d}{k} \right )$. Then  
\[ \Prob{ \exists I \in [d], \ |I| = k, \ \| X R_I\| > C_1 (1+\eps) \sqrt{\beta}  \norm{X}_\HS} \leq \exp\left (\frac{-c \eps^2 \beta \sr(X)}{K^4} \right ).\]
\end{lemma}
\begin{proof}
Let $Q$ be a $p \times k$ matrix drawn from the same distribution as $R$. From Lemma~\ref{lem:sparsefixed} for a fixed $\v \in S^{k-1}$
\[ \Pr \left [ \left | \|XQ\v\| - \sqrt{\beta} \|X\|_F  \right | > (1+\eps) \sqrt{\beta} \| X \|_F \right ] \leq 2\exp \left (\frac{-c_1 \beta \eps^2 \sr(X)}{K^4} \right ).\]
Let $\MM$ be a $1/2$-net on $S^{k-1}$ in the Euclidean metric. Standard volumetric argument gives that $|\MM| \leq 5^k$~\citep{V11}. By a union bound, with probability at least
\[ 5^k \cdot 2\exp \left (\frac{-c_1 \eps^2 \beta \sr(X)}{K^4} \right ), \]
every $\v \in \MM$ satisfies $\| X Q \v\| \leq  (1+\eps) \sqrt{\beta}  \| X\|_F$. By our choice of $\sr(X)$, $5^k \cdot 2\exp(-c_1 \eps^2 \beta \sr(X)/K^4) \leq \exp(-c_2 \eps^2 \beta \sr(X)/K^4)$. Let us condition on this probability event happening. Since every $\v \in S^{k-1}$ can be written as $\v = \a+\b$, where $\a \in \MM$ and $\|\b\| \leq 1/2$, we get
\[\|XQ\| = \max_{\v \in S^{k-1}} \| XQ \v\| \leq \max_{\a \in \MM} \|XQ\a\| + \max_{\b:\|\b\| \leq 1/2} \|XQ\b\| \leq (1+\eps) \sqrt{\beta}  \| X\|_F + \frac1 2 \| X Q\| \Longrightarrow \|X Q\| \leq 2(1+\eps) \sqrt{\beta}  \| X\|_F.\]
Using this along with a union bound over all $p \times k$ submatrices of $R$, we obtain
\begin{align*}  
 &\Pr \left[\exists I \subset [d], |I|=k \ \norm{XR_I} > C_1 (1+\eps) \sqrt{\beta} \norm{X}_{\HS} \right] \leq  \binom{d}{k} \cdot \exp \left (\frac{-c_2 \beta \eps^2 \sr(X)}{K^4} \right ) \leq \exp\left (\frac{-c_3 \beta \eps^2 \sr(X)}{K^4}\right ),
\end{align*}
by our assumption on $\sr(X)$. 
\end{proof}

\begin{theorem}~\label{thm:sparse}
Let $X$ be an $n \times p$ matrix. Let $\hat{R} = \Omega \odot \Gamma$ be a $p \times d$ matrix, where $\Omega \in \R^{p \times d}$ is an i.i.d.\ Bernoulli matrix where each entry is $1$ independently with probability $\beta$, and $\Gamma=(\Gamma_{ij})$ is an random matrix with independent entries that satisfy: $\E[\Gamma_{ij}] = 0$, $\E[\Gamma_{ij}^2] = 1$, and $\|\Gamma_{ij}\|_{\psi_2} \leq K$. Let $R = \hat{R}/(c_0\sqrt{\beta})$. Let $\epsilon \in (0,1)$, and let $k \in \N$ be a number satisfying $\sr(X)\ge CK^4 \frac{k}{\beta \epsilon^2} \log \left ( \frac{d}{k} \right )$.
Then with probability at least $1-\exp(-c \beta \epsilon^2 \sr(X)/K^4)$, the matrix $XR$ satisfies $(k, \epsilon)$-RIP, i.e.,
\[ \forall \u \in \Sigma_k, \; (1-\epsilon) \norm{X}_{\HS} \le \norm{XR\u} \le (1+\epsilon)  \norm{X}_{\HS}. \]
\end{theorem}
\begin{proof}
From Lemma~\ref{lem:sparsefixed},
\[ \Pr \left [ \left | \|X\hat{R}\u\| - \sqrt{\beta} \|X\|_F \right | > \epsilon \sqrt{\beta} \| X \|_F \right ] \leq 2\exp \left (\frac{-c_1 \beta \epsilon^2 \sr(X)}{K^4} \right ).\] 
We now use a standard net argument. Let $\NN$ be an $(1/2C_2)$-net $\NN$ over $\Sigma_k$ as  in Theorem~\ref{thm:rand}. Taking the union bound over this net, we obtain
\begin{align*}
\Pr \left[ \forall \u \in \NN, \ \ | \norm{X\hat{R}\u} - \sqrt{\beta} \norm{X}_{\HS} | \le \epsilon \sqrt{\beta} \norm{X}_{\HS} \right] & \ge 1- \exp \left( \frac{-c_1 \beta \epsilon^2 \sr(X)}{K^4} + k \log \frac{C_0d}{k} \right) \\
& \geq 1-\exp\left(\frac{-c_2 \beta \epsilon^2 \sr(X)}{K^4} \right ),
\end{align*}
where we used the assumption on $\sr(X)$. Now using the spectral norm bound from Lemma~\ref{lem:sparsenrom} in conjunction with the approximation idea used in Theorem~\ref{thm:rand} yields that with probability at least $1-\exp(-c \beta \epsilon^2 \sr(X)/K^4)$, $c_0 (1-\eps)\sqrt{\beta}  \norm{X}_{\HS} \leq  \|X\hat{R}\u\| \le c_0 (1+\eps)\sqrt{\beta}  \norm{X}_{\HS}$ for all $\u \in \Sigma_k$. We finish by noting that $R=\hat{R}/(c_0 \sqrt{\beta})$.
\end{proof}

\subsection{Restricted Isometry of $XR$ under Convex Concentration Property on $R$} \label{sec:ccp}
We investigate the case where $R \in \R^{p \times d}$ is composed of independent columns satisfying the convex concentration property. In this case, we utilize the recent result of~\citep{adamczak2015note}, who proved the Hanson-Wright inequality for isotropic random vectors having convex concentrations property. 

\begin{theorem} [Theorem 2.3~\citep{adamczak2015note}] \label{thm:ccpHS}
Let $\x = (x_1,\dots,x_d) \in \R^d$ be a random vector with independent components $x_i$ which satisfy $\E[x_i] = 0$. Let $\x$ satisfy the convex concentration property with constant $K$. Let $M$ be a $d\times d$ matrix. Then for every $t \ge 0$,
$$\Pr[|\x^\top M \x - \E[\x^\top M \x]| > t ] \leq 2 \exp \left ( -c \min \left \{\frac{t^2}{2K^4 \| M \|_F^2}, \frac{t}{K^2 \| M \|}  \right \} \right ).$$
\end{theorem}

Our first result here is to show that a linear combination of independent vectors having a convex concentration property would have this property as well (Lemma~\ref{lem:ccp}). To this end, we start with a fixed $\u=(u_1,\dots,u_d) \in S^{d-1}$ and construct a martingale with variables $\E [ \phi ( \sum_{i=1}^d \r_i u_i  ) \,|\, \r_1,\dots,\r_j  ]$ where $\phi \, : \, \R^p \rightarrow \R$ is a 1-Lipschitz  convex function and $\r_i$ is the $i$th column in $R$, and then apply Azuma's inequality to it. Once we have this established, verifying RIP consists of applying the Hanson-Wright inequality of~\citep{adamczak2015note} in a proof framework similar to Theorem~\ref{thm:rand}. 

\begin{lemma} \label{lem:ccp}
Let $R$ be a $p \times d$ random matrix with independent columns satisfying the convex concentration property with constant $K$. Let $\u \in S^{d-1}$. Then $R\u$ is a random vector satisfying the convex concentration property with constant $O(K)$.
\end{lemma}
\begin{proof}
Let $\u=(u_1,\dots,u_d)$. Let $\phi\,:\, \R^p \rightarrow \R$ be a  1-Lipschitz convex function. Let $\r_1,\dots,\r_d \in \R^p$ be the columns of $R$. For any $j \in [d]$, the function $V_j\,:\, \R^p \rightarrow \R$ defined by
$$V_j(\a) = \phi \left ( \sum_{i \neq j} \r_i u_i + \a u_j \right ) - \E \left [ \phi \left ( \sum_{i \neq j} \r_i u_i + \a u_j \right ) \right ],$$
is also convex and Lipschitz with the Lipschitz constant $|u_j|$. The convex concentration property implies that the random variable
$$\Delta_j  = \E \left [\phi \left (\sum_{i=1}^d \r_i u_i \right)  \,|\, \r_1,\dots,\r_j \right ] - \E \left [\phi \left (\sum_{i=1}^d \r_i u_i \right)  \,|\, \r_1,\dots,\r_{j-1} \right ]$$
is subgaussian with $\psi_2$-norm $\| \Delta \|_{\psi_2} \leq K |u_j|$. To verify it, we can condition on $\r_i, i \neq j$, and then integrate over $\r_1,\dots,\r_{j-1}$ using Jensen's inequality for the $\psi_2$-norm. The random variables
$$m_j = \E \left [ \phi \left ( \sum_{i=1}^d \r_i u_i \right ) \,|\, \r_1,\dots,\r_j  \right ], \;\;\; j \in \{0,\dots,d\},$$
form a martingale with the martingale differences $\Delta_j$. Hence, Azuma's inequality implies that
$$m_d - m_0 = \phi \left ( \sum_{i=1}^d \r_i u_i \right ) - \E \left [ \phi \left ( \sum_{i=1}^d \r_i u_i \right ) \right]$$
is subgaussian random variable with  $\| m_d - m_0 \|_{\psi_2} \leq C K \| \u \| = O(K)$. This completes the proof.
\end{proof}

\begin{theorem}~\label{thm:ccp}
Let $X$ be an $n \times p$ matrix. Let $R$ be a $p \times d$ matrix with mean zero independent columns satisfying the convex concentration property (Definition~\ref{defn:ccp}) with constant $K$. Let $\epsilon \in (0,1)$, and let $k \in \N$ be a number satisfying $\sr(X)\ge CK^4 \frac{k}{\epsilon^2} \log \left ( \frac{d}{k} \right )$.
Then with probability at least $1-\exp(-c \epsilon^2 \sr(X)/K^4)$, the matrix $XR$ satisfies $(k, \epsilon)$-RIP, i.e.,
 \[ \forall \u \in \Sigma_k, \; (1-\epsilon) \norm{X}_{\HS} \le \|XR\u\| \le (1+\epsilon) \norm{X}_{\HS}. \]
\end{theorem}
\begin{proof}
Fix any $\u \in S^{d-1}$. From Lemma~\ref{lem:ccp}, we get that $R\u$ satisfies the c.c.p with constant $O(K)$. Using $R\u$ as the random vector, a simple consequence of the Hanson-Wright inequality from Theorem~\ref{thm:ccpHS} is for all $t\ge 0$,
$$\Pr[ | \| XR \u \| - \| X \|_F| \geq t]\leq 2\exp \left (\frac{-c_1 t^2}{K^4 \| X \|^2} \right ).$$
Setting $t = \eps \| X \|_F$, we get
\begin{align} \label{eqn:cc}
\Pr[ | \| XR \u \| - \| X \|_F| \geq \eps \| X \|_F ]\leq 2\exp\left(\frac{-c_1 \eps^2 \sr(X)}{K^4} \right ).
\end{align}
The remainder of the proof follows the framework of Theorem~\ref{thm:rand}. From the Hanson-Wright inequality from Theorem~\ref{thm:ccpHS}, we obtain that
$$\Prob{ \exists I \in [d], \ |I| = k, \ \| X R_I\| > C_1  \eps \norm{X}_\HS}  \le 2 \exp(-c_0 \eps^2 \sr(X)/K^4).$$
Verifying RIP consists now of applying~\eqref{eqn:cc} to a net in the set of sparse vectors on the unit sphere as in Theorem~\ref{thm:rand} and the details are omitted here.
\end{proof}

\section{Missing Details from Section~\ref{sec:had}} \label{app:had}
\newenvironment{pr2}{{\noindent\it {\bf Proof of Lemma~\ref{lem: psi_2}}.
}}{\hfill{\qed}\vspace{1ex}}
\begin{pr2}
We start with the decoupling argument as in~\citep{RVHanson-Wright}. Let $\delta_1,\dots,\delta_d$ be i.i.d.\ Bernoulli$(1/2)$ random variables and denote by $I$ the random set $I=\{i \in [d]: \delta_i=1\}$. Let $a>0$, and assume that for any realization of $I$,  
\[
 \E\left [\exp \left( \frac{1}{2a} \sum_{i \in I \, j \notin I} r_i r_j u_{ij} \right) \right] \le 2.
\]
Then by Jensen's inequality,
\begin{align*}
\E \left[ \exp \left( \frac{1}{2a} \sum_{(i,j) \in \binom{[d]}{2} } r_i r_j u_{ij} \right) \right ]
 & \le  \E \left [\exp \left(\frac{1}{2a}  \sum_{(i,j) \in \binom{[d]}{2} } 2 \E_{\delta} \left[ \delta_i (1-\delta_j) r_i r_j u_{ij}  \right] \right ) \right ] \\
 &  \le \E \left[ \E_{\delta} \left[ \exp \left(\frac{1}{a}  \sum_{(i,j) \in \binom{[d]}{2} }  \delta_i (1-\delta_j) r_i r_j u_{ij} \right) \right] \right ] \\
 &\le  \max_{I \subset [d]}  \E \left [ \exp \left( \frac{1}{2a} \sum_{i \in I \, j \notin I} r_i r_j u_{ij} \right)  \right ]\le 2.
\end{align*}
This calculation means that 
\begin{equation} \label{eq: dec}
 \norm{\sum_{(i,j) \in \binom{[d]}{2} } r_i r_j u_{ij}}_{\psi_2} 
 \le \max_{I \subset [d]}  \norm{\sum_{i \in I \, j \notin I} r_i r_j u_{ij} }_{\psi_2}.
\end{equation}
From now on, we fix $I \subset [d]$ and focus on bounding $\norm{\sum_{i \in I \, j \notin I} r_i r_j u_{ij} }_{\psi_2}$.

We will decompose the random variable $\sum_{i \in I \, j \notin I} r_i r_j u_{ij}$ into the sum of at most $4\sqrt{k}$ subgaussian random variables corresponding to disjoint parts of $\supp(\u)$ and bound the $\psi_2$-norms of these variables separately.
To this end, we introduce the following induction procedure. Let $P_1, P_2$ be coordinate projections of $I \times I^{\co}$ on the first and the second coordinates respectively. 
Set $F_1 = \supp(\u)$ and let $I_1 = P_1( F_1)$ and $J_0 = P_2( F_1)$. Choose a set $U_1 \subset F_1$ such that $P_1 (U_1)=I_1$ and for any $i \in I_1$, $|P_1^{-1}(i) \cap U_1|=1$. This means that we lift the projection $I_1$ in the set $F_1$ choosing a single point in $F_1$ for any point in the projection. The set $U_1$ can and will be chosen to contain the set $F_1 \cap (I \times \{j\})$ for some $j \in J_0$.
If $F_1=U_1$, we set $W_1=\varnothing$ and stop. Otherwise,  set $E_1=F_1 \setminus U_1$, and $J_1=P_2(E_1)$. By construction, $J_1 \subset J_0$, and $|J_1| \le |J_0|-1$. Similarly to what we have done before, choose a set $W_1 \subset E_1$ such that $P_2 (W_1)=J_1$ and for any $j \in J_1$, $|P_2^{-1}(j) \cap W_1|=1$. We choose the set $W_1$ so that it contains the set $E_1 \cap (\{i\} \times I^{\co})$ for some $i \in I_1$. If $E_1=W_1$, we stop. 
Otherwise, set $F_2=E_1 \setminus W_1$, and $I_2=P_1(F_2)$. Again by construction, $F_2 \subset F_1$, $I_2 \subset I_1$, and $|I_2| \le |I_1|-1$. This finishes the first step of the induction. After that, we repeat this process using the set $F_2$ in place of $F_1$. 
At each step of the process, we create disjoint sets $U_1,U_2, \ldots$ and $W_1,W_2, \ldots$ contained in $\supp(\u)$.
Since these sets  are non-empty, this process eventually stops. After it stops, we obtain a decomposition
\[ \supp(\u) \cap ( I \times I^{\co})= \bigcup_{t=1}^s U_t \cup  \bigcup_{t=1}^s W_t,\]
where all terms are disjoint. Moreover, the construction yields $I_1 =P_1(U_1) \supset I_2=P_1(U_2) \supset \cdots \supset P_1(U_s)$ with $|U_1|>|U_2|> \cdots >|U_s|$. Similarly, $J_1 =P_2(W_1) \supset J_2=P_1(W_2) \supset \cdots \supset P_2(W_s)$ with $|W_1|>|W_2|> \cdots >|W_s|$.
In view of the decomposition above, one of the quantities $\sum_{j=1}^s |U_j|$ or $\sum_{j=1}^s |W_j|$ is at most $|\supp(\u) \cap ( I \times I^{\co})|/2$. Assume that this is the former one. Then
\[
\frac{k}{2}
\ge \frac{|\supp(\u) \cap ( I \times I^{\co})|}{2} 
\ge \sum_{t=1}^s |U_t|
\ge \sum_{t=1}^s \left( |U_s|+s-t \right) 
\ge \frac{s(s-1)}{2},
\]
which implies $s \le \sqrt{k}+1$. The estimate in the other case is exactly the same. 

Let us use the decomposition we constructed to estimate $ \norm{\sum_{i \in I \, j \notin I} r_i r_j u_{ij} }_{\psi_2}$. For $t \in [s]$, denote 
\[ v_t= \sum_{(i,j) \in U_t} r_i r_j u_{ij}  \quad \text{and} \quad w_t= \sum_{(i,j) \in W_t} r_i r_j u_{ij}. \]
The random variables $r_i$ and $r_j$ are independent for all  $i \in I$ and $j \in I^{\co}$. Hence, by Hoeffding's inequality,
\begin{align*}
\norm{v_t}_{\psi_2} 
&\le \sup_{|a_1|, \ldots, |a_s| \le K} \norm{\sum_{(i,j) \in U_t} r_i a_j u_{ij} }_{\psi_2} 
\le 4\tau \sup_{|a_1|, \ldots, |a_s| \le K} \left( \sum_{(i,j) \in U_t} a_j^2 u_{ij} ^2 \right)^{1/2} \le 4\tau^2 \left( \sum_{(i,j) \in U_t} u_{ij} ^2 \right)^{1/2},\
\end{align*}
and $\norm{w_t}_{\psi_2}$ satisfies the same estimate. 
Finally,
\begin{align*}
\norm{\sum_{i \in I \, j \notin I} r_i r_j u_{ij} }_{\psi_2}
&\le \sum_{t=1}^s \norm{v_t}_{\psi_2}  + \sum_{t=1}^s \norm{w_t}_{\psi_2} \le 4\tau^2 \sum_{t=1}^s \left( \sum_{(i,j) \in U_t} u_{ij} ^2 \right)^{1/2} + 4\tau^2  \sum_{t=1}^s  \left( \sum_{(i,j) \in W_t} u_{ij} ^2 \right)^{1/2} \\
&\le 8\tau^2 \sqrt{s} \left( \sum_{i \in I \, j \notin I} u_{ij} ^2 \right)^{1/2}
\le 8\tau^2 \sqrt{s} \norm{\u}.
\end{align*}
Combining this with \eqref{eq: dec} and taking into account that $s \le \sqrt{k}+1$, we complete the proof of the lemma.
\end{pr2}

\begin{theorem} [Theorem~\ref{thm: RIP2} Restated]
Let $X$ be an $n \times p$ matrix, and let $R$ be a $p \times d$ random matrix with independent entries $R_{ij}$ such that $\E[R_{ij}]=0, \E[R_{ij}^2]=1, \text{and } |R_{ij}| \le \tau \text{ almost surely}$.
Let $\ell \geq 3$ be a constant. Let $\eps \in (0,1)$, and let $k \in \NNN$ be a number satisfying $\sr(X) \ge \frac{C \tau^{4\ell}k^3}{\eps^2}  \log \left( \frac{d^\ell}{k} \right)$.  Then with probability at least $1- \exp(-c \eps^2 \sr(X)/(k^2\tau^{4\ell}))$, the matrix $XR^{(\ell)}$ satisfies the $(k,\eps)$-RIP property, i.e.,  for any $\u \in S^{\binom{d}{\ell}-1}$ with $|\supp(\u)| \le k$,
\[(1-\eps) \norm{X}_{\HS} \le \norm{XR^{(\ell)}\u} \le  (1+\eps) \norm{X}_{\HS}.\]
\end{theorem}
\begin{proof}
Let $\u \in \R^{\binom{d}{\ell}}$ be a vector with $|\supp(\u)| \le k$. Let $u_{i_1\dots i_\ell}$ be the $(i_1,\dots,i_\ell)$th element in $\u$ (with $1\leq i_1 <i_2 \dots < i_\ell \le d$). Let $l \in [p]$, and define 
\[y_l=\sum_{(i_1,\dots,i_\ell) \in \binom{[d]}{\ell} } R_{li_1} R_{li_2} \dots R_{li_\ell} u_{i_1\dots i_\ell}.\]
The random variables $y_l, \ l \in [p]$ are independent. From triangle inequality,
\begin{align} \label{eqn:newtriangle}
\| y_l \|_{\psi_2} & = \norm{\sum_{(i_1,\dots,i_\ell) \in \binom{[d]}{\ell} }   R_{li_1} R_{li_2} \dots R_{li_\ell} u_{i_1\dots i_\ell} }_{\psi_2} \le \sum_{(i_1,\dots,i_\ell) \in \supp(\u)}  \norm{R_{li_1} R_{li_2} \dots R_{li_\ell} u_{i_1\dots i_\ell} }_{\psi_2} \nonumber  = O(\tau^{\ell} \norm{\u}_1) \\ & = O(\tau^{\ell} \sqrt{k} \norm{\u}).
\end{align}
The rest of the proof proceeds as in Theorem~\ref{thm: RIP}, by using the  $\| y_l \|_{\psi_2}$ bound from~\eqref{eqn:newtriangle}. In this case, we use a net $\NNN$ on the set of $k$-sparse vectors in $S^{\binom{d}{\ell}-1}$.
\end{proof}

\subsection{Compressed Learning Background} \label{app:comp}
We discuss some background on compressed learning.
Consider the standard binary classification task with labeled data from $\mathcal{X} = \{(\x, \upsilon) : \x \in \R^d , \x \mbox{ is $k$-sparse }, \|\x\| \leq \alpha, \upsilon \in \{-1, 1\}\}$, with $\upsilon$ indicating the label on $\x$. We focus on linear loss functions having the form, $f(\langle \x,\theta \rangle;\upsilon)$ for $\theta \in \R^d$, where $f \,: \,  \R \times \{-1,1\} \rightarrow \R$ is assumed to be convex and Lipschitz in the first parameter.  This type of program captures a variety of important learning problems, e.g., the linear regression is captured by setting $f(\langle \x, \theta \rangle; \upsilon) = (\upsilon - \langle \x,\theta \rangle)^2$, logistic regression is captured by setting $f(\langle \x, \theta \rangle; \upsilon) = \ln(1+\exp(-\upsilon\langle \x,\theta \rangle))$, support vector machine is captured by setting $f(\langle \x, \theta \rangle; \upsilon) = \hinge(\upsilon \langle \x,\theta \rangle)$, where $\hinge(a) = 1-a$ if $a \leq 1$ and $0$ otherwise.

For a distribution $\mathcal{D}$ over $\mathcal{X}$, define, $f_\mathcal{D}(\theta) = \E_{(\x,\upsilon) \sim \mathcal{D}}[f(\langle\x,\theta\rangle;\upsilon)]$ as the generalization loss on the original domain for parameter $\theta$. Let $\theta^\star \in \mbox{argmin}_{\theta \in \\R^d}\, f_\mathcal{D}(\theta)$. The following theorem from~\citet{arora2018compressed}, building upon a result by~\citep{calderbank2009compressed}, relates the generalization loss obtained by solving the classification problem on the compressed domain to the loss on the original domain.
\begin{theorem} [Restated from Theorem 4.2~\citep{arora2018compressed}] \label{thm:arloss}
Let $M \in \R^{n \times d}$ be a $(k,\eps)$-RIP matrix. Let  $(\x_1,\upsilon_1),\dots,(\x_b,\upsilon_b)$ be a set of labeled data drawn i.i.d.\ from a distribution $\mathcal{D}$ over $\mathcal{X}$. Let $f$ be a $\lambda$-Lipschitz convex linear loss function and $\theta^\star \in  \mbox{argmin}_{\theta \in \R^d}\, f_\mathcal{D}(\theta)$ be its minimizer over $\mathcal{D}$. Let $\hat{\vartheta} \in \R^n$ be the minimizer of a classifier trained on $(M\x_1,\upsilon_1),\dots,(M\x_b,\upsilon_b)$ minimizing the $L_2$-regularized empirical loss function $\frac{1}{b} \sum_i f(\langle M\x_i,\vartheta \rangle;\upsilon_i) + \frac{1}{2C} \| \vartheta\|^2$, i.e., $\hat{\vartheta} \in \mbox{argmin}_{\vartheta \in \R^n}\, \frac{1}{b} \sum_i f(\langle M\x_i,\vartheta \rangle;\upsilon_i) + \frac{1}{2C} \| \vartheta\|^2$ (for appropriate constant $C$), then with probability at least $1-\delta$, we have
$$f_\mathcal{D}(\hat{\vartheta}) \leq f_\mathcal{D}(\theta^\star) + O \left ( \lambda \alpha \| \theta^\star\| \sqrt{\epsilon + \frac{1}{b}\log\frac 1 \delta} \right ),$$ 
where $f_\mathcal{D}(\vartheta) =  \E_{(\x,\upsilon) \sim \mathcal{D}}[f(\langle M\x,\vartheta\rangle;\upsilon)]$ is the generalization loss on the compressed domain for parameter~$\vartheta$.
\end{theorem}

\end{document}